\documentclass[times, review, 10pt]{elsarticle}

\usepackage{amsmath,amssymb,amsthm}
\usepackage{graphicx}

\makeatletter
\let\DDCLorigincludegraphics\includegraphics
\renewcommand{\includegraphics}[2][]{%
  \IfFileExists{#2}{%
    \DDCLorigincludegraphics[#1]{#2}%
  }{%
    \begingroup
    \setlength{\fboxsep}{0pt}%
    \fbox{\begin{minipage}[c][0.18\textheight][c]{0.9\linewidth}
    \centering\footnotesize Missing figure file: \texttt{#2}
    \end{minipage}}%
    \endgroup
  }%
}
\makeatother
\usepackage{booktabs}
\usepackage{multirow}
\usepackage{hyperref}
\usepackage{microtype}
\usepackage{enumitem}
\usepackage{mathtools}
\usepackage{xcolor}
\usepackage[ruled,vlined,linesnumbered]{algorithm2e}
\usepackage[top=4.3cm, right=4.8cm, bottom=4.3cm, left=4.8cm]{geometry}

\usepackage{etoolbox}
\usepackage{setspace}
\newtheorem{theorem}{Theorem}
\newtheorem{proposition}{Proposition}
\newtheorem{corollary}{Corollary}

\newtheorem{definition}{Definition}
\newtheorem{fact}{Fact}
\newtheorem{remark}{Remark}

\newcommand{\Lq}{\mathcal{L}_q}
\newcommand{\LOLS}{L_{\mathrm{OLS}}}
\newcommand{\zn}{\mathbf{z}_n}
\newcommand{\pk}{\mathbf{p}_k}
\newcommand{\mun}{\boldsymbol{\mu}_n}
\newcommand{\qnk}{q_{nk}}
\newcommand{\R}{\mathbb{R}}
\newcommand{\norm}[1]{\left\|#1\right\|}
\newcommand{\SP}{\mathcal{S}(P)}
\newcommand{\HQ}{H(Q)}
\newcommand{\Sigmaq}{\Sigma_q}

\begin{document}

\begin{frontmatter}

\title{Collapse-Free Prototype Readout Layer for Transformer Encoders}

\author[lti]{Giansalvo Cirrincione}
\author[cdu]{Rahul Ranjeev Kumar\corref{cor}}
\cortext[cor]{Corresponding author}
\affiliation[lti]{organization={Laboratory LTI, Universit\'e de Picardie Jules Verne},
                  city={Amiens},
                  country={France },
                  ead={exin@u-picardie.fr}}
\affiliation[cdu]{organization={Institute of Energy and Resources, Charles Darwin University},
                  city={Darwin, NT},
                  country={Australia }, ead={rahul.kumar@cdu.edu.au}}

\begin{abstract}
Transformer encoders produce rich token representations, but extracting a
compact, structured summary from them typically relies on simple heuristics
such as averaging or taking a single \emph{class token} --- operations that
discard information and provide no training time feedback on representational
quality.
This paper introduces \textbf{DDCL-Attention}, a prototype based competitive
readout layer that replaces such pooling heuristics with a principled
compression mechanism.
The key idea is to maintain a small bank of globally learned \emph{prototype
vectors} --- reference representations that summarise the recurring patterns
in the data --- and to assign each token to these prototypes via a soft,
probabilistic rule.
The layer output is a weighted combination of prototypes, one per token,
and operates at linear complexity in sequence length rather than the quadratic
cost of standard self-attention.

Three practical advantages distinguish DDCL-Attention from existing
prototype-based mechanisms such as Slot Attention and Perceiver.
First, it provides a mathematical \emph{guarantee against prototype collapse}:
an exact algebraic decomposition of the training loss into a reconstruction
term and a diversity term ensures that prototypes cannot all converge to the
same point, a common failure mode that renders the prototype bank useless.
Second, \emph{training stability is proved formally}: under the practical
condition that prototypes are updated faster than the encoder, the joint
training dynamics are shown to be stable via Tikhonov's singular perturbation
theory, with explicit conditions on the learning rate ratio.
Third, the layer is \emph{versatile}: the same mechanism instantiates three
distinct application paradigms --- a final readout layer, a differentiable
codebook generalising VQ-VAE, and a hierarchical document compressor ---
each with its own theoretical motivation.

Experiments across four datasets confirm that the decomposition holds with
zero violations in all settings, that prototype separation grows as predicted
by theory when the stability condition is satisfied, and that the codebook
achieves full utilisation (100\%) compared to 39\% for standard hard
vector quantization.
An additional experiment on orbital debris classification demonstrates
applicability to scientific tabular data beyond the standard NLP and
vision benchmarks.
\end{abstract}


\end{frontmatter}

\section{Introduction}
\label{sec:intro}

The transformer architecture~\cite{Vaswani2017} and its self-attention
mechanism have become the dominant paradigm in sequence modelling, vision,
and multimodal learning.
Self-attention computes pairwise interactions between all $T$ tokens ($T$ being the sequence length) at
$O(T^2)$ cost, and while efficient variants reduce this cost ---
Reformer~\cite{Kitaev2020} via locality-sensitive hashing,
Linformer~\cite{Wang2020} via low-rank projection,
FlashAttention~\cite{Dao2024} via IO-aware tiling, and linear
attention formulations surveyed in~\cite{Tay2023} ---
they do so by approximating or sparsifying the attention
matrix rather than by rethinking the underlying similarity function.
Modern self-supervised vision transformers such as
DINOv2~\cite{Oquab2024} illustrate how powerful such representations
can be, yet their readout mechanisms still rely on simple CLS-token
extraction or global average pooling.

A parallel line of research has revisited the role of
\emph{prototypes}~\cite{kohonen,rumelhart} --- globally learned reference
vectors that represent recurring patterns in the data.
Slot Attention~\cite{Locatello2020} uses a fixed set of slots as dynamic
keys and values, updated iteratively via a GRU at inference time.
Perceiver~\cite{Jaegle2021} projects a long input sequence onto a short
latent array via cross-attention, achieving linear $O(TK)$ complexity;
its successor Perceiver IO~\cite{Jaegle2022} further generalises
this to structured outputs.
Recent extensions such as the Slot Mixture Module~\cite{Kirilenko2024},
which replaces soft $k$-means with a Gaussian mixture model,
the Adaptive Slot Attention mechanism~\cite{Fan2024},
which dynamically adjusts the number of slots,
and the Prototype Transformer (ProtoT)~\cite{Yordanov2026},
which integrates prototype routing at every layer of an autoregressive
language model, explore richer prototype representations but still lack
formal anti-collapse and stability guarantees.
Neither the original Slot Attention framework nor these extensions
provides a theoretical guarantee against prototype
collapse, and all require design choices --- iterative slot refinement,
special input pre-processing --- that resist systematic stability analysis.

The DDCL framework~\cite{Cirrincione2026} provides exactly such a
guarantee for prototype based clustering: the exact loss decomposition
$\Lq = \LOLS + V$ (defined formally in Section~\ref{sec:ddcl_bg}) implies a separation force
$\nabla_P V = 2P\Sigmaq$ that makes prototype collapse a locally
unstable saddle.
The stability analysis of~\cite{Cirrincione2026} is, however,
limited to the \emph{frozen encoder reduced system}: the joint stability
of a coupled encoder--prototype system under simultaneous gradient
updates remains an open problem explicitly identified therein.

The present paper closes this gap and extends the framework to practical
transformer settings.
This paper introduces \textbf{DDCL-Attention}, a prototype based competitive
readout layer that maps token embeddings to soft centroid representations
via Boltzmann assignments over a global prototype bank.
DDCL-Attention is not a general replacement for self-attention, but a
\emph{complementary module}: self-attention models intra-sequence
dependencies, while DDCL-Attention compresses encoder output into a
structured prototype vocabulary.
The natural deployment is as the final readout layer of a transformer
stack, replacing CLS-token pooling or global average pooling.

Beyond the stability framework, it is demonstrated that DDCL-Attention
instantiates three distinct and practically relevant paradigms in
transformer architectures, each with its own theoretical motivation
and independent empirical validation.

\noindent
The main contributions of this paper are:

\begin{enumerate}[leftmargin=1.8em]

  \item \textbf{DDCL-Attention layer} (Section~\ref{sec:layer}):
  a prototype based competitive layer with $O(TK)$ complexity ($K$ prototypes),
  multi-head extension, and residual connection that integrates
  directly into standard transformer pipelines without iterative
  inference time updates.

  \item \textbf{Exact loss decomposition for coupled systems}
  (Section~\ref{sec:decomp}):
  the identity $\Lq = \LOLS + V$ holds exactly for any differentiable
  encoder; $V \geq 0$ acts as an implicit anti-collapse force and
  the encoder gradient $\mathbf{g}_n = 2(\mathbf{z}_n - \boldsymbol{\mu}_n)$, where $\mathbf{z}_n$ is the token embedding and $\boldsymbol{\mu}_n$ its soft centroid, is identified as a compression signal.

  \item \textbf{Time-scale stability theorem}
  (Section~\ref{sec:timescale}):
  under $\varepsilon = \eta_\theta / \eta_P \ll 1$, where $\eta_\theta$
  and $\eta_P$ are the encoder and prototype learning rates respectively,
  the coupled encoder--prototype dynamics reduce via Tikhonov's theorem
  to a Lyapunov-stable fast prototype subsystem and a slow encoder;
  explicit sufficient conditions for joint stability are derived.

  \item \textbf{Global free energy Lyapunov analysis}
  (Section~\ref{sec:lyapunov}):
  a global Lyapunov function proves convergence to configurations
  with strictly positive prototype separation for any $\varepsilon$
  and any monotone annealing schedule.

  \item \textbf{DDCL as differentiable vector quantization}
  (Section~\ref{sec:vq}):
  $L_\mathrm{OLS}$ is the soft commitment loss and $V$ is the codebook
  diversity term; gradient flow through $q_{nk}$ eliminates the
  straight-through estimator and provably prevents dead codes.

  \item \textbf{Hierarchical decomposition}
  (Section~\ref{sec:hierarchical}):
  for a stack of $L$ layers,
  $\Lq^\mathrm{total} = \sum_{\ell=1}^L \Lq^{(\ell)}$
  with $V^{(\ell)} \geq 0$ and the anti-collapse force active
  simultaneously at every level.

  \item \textbf{Empirical validation on three paradigms}
  (Section~\ref{sec:experiments}):
  (i)~final readout on SST-2, IMDB, 20NG (frozen BERT, $\varepsilon=0.1$);
  (ii)~soft VQ-VAE on CIFAR-10 --- 100\% vs.\ 39\% codebook utilisation;
  (iii)~hierarchical compression on 20NG, $V^{(1)}\!\geq 0$ and
  $V^{(2)}\!\geq 0$ confirmed simultaneously across all epochs.
  Zero decomposition violations in all settings.

\end{enumerate}

The remainder of the paper is organised as follows.
Section~\ref{sec:background} recalls the DDCL decomposition and
self-attention background.
Section~\ref{sec:layer} defines the DDCL-Attention layer.
Sections~\ref{sec:decomp}--\ref{sec:lyapunov} develop the stability theory.
Section~\ref{sec:theoretical_connections} establishes the VQ and hierarchical
connections.
Section~\ref{sec:comparison} compares DDCL-Attention with Slot~Attention and
Perceiver.
Section~\ref{sec:experiments} reports empirical validation.
Section~\ref{sec:discussion} discusses contributions and limitations.
Section~\ref{sec:conclusions} concludes.

\section{Background}
\label{sec:background}

\subsection{Transformer self-attention}

Given an input sequence of $T$ tokens represented as a matrix
$X \in \R^{T \times d}$ (where $d$ is the embedding dimension),
standard self-attention computes:
\begin{align}
  \mathrm{Attn}(Q,K,V)
  &= \mathrm{softmax}\!\left(\frac{QK^\top}{\sqrt{d_k}}\right)V,
  \nonumber\\
  &Q = XW_Q,\quad K = XW_K,\quad V = XW_V,
  \label{eq:selfattn}
\end{align}
where $d_k$ is the key dimension.
The scalar $a_{nt}$ denotes the attention weight between query token $n$
and key token $t$: it is the $(n,t)$ entry of the softmax matrix
$\mathrm{softmax}(QK^\top/\sqrt{d_k})$, and measures how much token $n$
``attends to'' token $t$ based on their dot product similarity.
Because both $Q$ and $K$ are computed from the input $X$, these weights
change at every forward pass --- keys and values are \emph{dynamic}
(sequence-dependent).

\subsection{DDCL and the loss decomposition}
\label{sec:ddcl_bg}

The DDCL competitive loss~\cite{Cirrincione2026} is defined over a
set of $N$ embeddings $\{\zn\}_{n=1}^N \subset \R^m$ (where $m$ is
the prototype dimension) and a bank of $K$ prototypes
$P = \{\pk\}_{k=1}^K \subset \R^m$:
\begin{equation}
  \Lq = \sum_n \sum_k \qnk \norm{\zn - \pk}^2,
  \qquad
  \qnk = \frac{\exp(-\norm{\zn-\pk}^2/T)}
              {\sum_{k'} \exp(-\norm{\zn-\pk'}^2/T)},
  \label{eq:Lq}
\end{equation}
where $T>0$ is a temperature parameter controlling the sharpness
of the assignments: high $T$ gives soft, nearly uniform assignments;
low $T$ gives hard, winner takes all assignments.
The central result of~\cite{Cirrincione2026} is the exact identity:
\begin{align}
  \boxed{\Lq = \LOLS + V}, \qquad
  &\LOLS = \sum_n \min_k \norm{\zn-\pk}^2, \nonumber\\
  &V = \sum_n\sum_k \qnk \norm{\pk-\mun}^2 \;\geq\; 0,
  \label{eq:decomp}
\end{align}
where $\mun = \sum_k \qnk \pk$ is the soft centroid.
Under stop gradient on assignments,
$\nabla_P V = 2P\Sigmaq$ where
$\Sigmaq = \sum_n(\mathrm{diag}(\mathbf{q}_n) -
\mathbf{q}_n\mathbf{q}_n^\top)$ is the aggregated soft assignment
covariance.
This gradient acts as a \emph{separation force}: prototype collapse
is a first order locally unstable saddle of~$\Lq$.

\begin{fact}[DDCL Lyapunov theorem~\cite{Cirrincione2026}]
\label{fact:lyapunov}
Under the regularised loss with $\lambda > 2\eta_P\binom{K}{2}$,
the reduced frozen encoder flow $\dot{P} = -\nabla_P\widetilde{\Lq}$
admits a global Lyapunov function and converges to the set
\[
  \mathcal{A} = \bigl\{P \;:\; \nabla_P\widetilde{\Lq}=0,\;
  \min_{j\neq k}\norm{\pk-\mathbf{p}_j}^2>0\bigr\}.
\]
\end{fact}

The present paper extends Fact~\ref{fact:lyapunov} to the
\emph{full coupled system} with simultaneously updated $(\theta, P)$.

\section{The DDCL-Attention Layer}
\label{sec:layer}

\subsection{Definition}
\label{sec:layer_def}

Let $\{\zn\}_{n=1}^T \subset \R^m$ be token embeddings from an upstream
encoder $f_\theta$, and $P = \{\pk\}_{k=1}^K \subset \R^m$ a bank of
$K$ globally learned prototypes, shared across all sequences.

\begin{definition}[DDCL-Attention]
\label{def:ddcl_attn}
The assignment weights are as in~\eqref{eq:Lq};
the output for token $n$ is:
\begin{equation}
  \mathbf{o}_n = \sum_{k=1}^{K} \qnk\, \pk = \mun,
  \label{eq:output}
\end{equation}
and the layer output with residual connection is:
\begin{equation}
  \mathbf{h}_n = \mathrm{LayerNorm}(\zn + W_O\, \mathbf{o}_n),
  \label{eq:residual}
\end{equation}
where $W_O \in \R^{m \times m}$ is a learnable output projection.
\end{definition}

\begin{remark}
$\mathbf{o}_n = \mun$ \emph{exactly}: the output is the soft centroid
of~\eqref{eq:decomp}, linking the layer definition directly to the
theoretical guarantees.
\end{remark}

Unlike self-attention, keys $\pk$ are \emph{global and static} within
a step --- they do not depend on the current input sequence.
Unlike Slot Attention, no iterative GRU update is performed at
inference time.

\subsection{Multi-head extension}
\label{sec:multihead}

For $H$ heads with per head dimension $m_h = m/H$, define $H$
independent prototype sets $\{P^{(h)}\}_{h=1}^H$ and input projections
$W_h \in \R^{m_h \times m}$:
\begin{equation}
  \mathbf{z}_n^{(h)} = W_h\zn, \quad
  q_{nk}^{(h)} = \mathrm{softmax}_k\!\left(
    -\tfrac{\|\mathbf{z}_n^{(h)}-\mathbf{p}_k^{(h)}\|^2}{T}
  \right), \quad
  \mathbf{o}_n^{(h)} = \textstyle\sum_k q_{nk}^{(h)}\,\mathbf{p}_k^{(h)}.
  \label{eq:multihead}
\end{equation}
The multi-head output concatenates the $H$ per-head soft centroids
into a single vector of dimension $Hm_h = m$, then applies the output
projection $W_O \in \R^{m \times m}$:
\begin{equation}
  \mathbf{o}_n = W_O\,\bigl[
    \mathbf{o}_n^{(1)\top} \;\cdots\; \mathbf{o}_n^{(H)\top}
  \bigr]^\top \;\in\; \R^m.
  \label{eq:multihead_out}
\end{equation}
Here $[\cdot]$ denotes column wise concatenation: the $H$ vectors
$\mathbf{o}_n^{(h)} \in \R^{m_h}$ are stacked into a single
$m$-dimensional vector before projection.
Each head independently attends to a different $m_h$-dimensional subspace
of the embedding, with its own prototype set $P^{(h)}$, promoting
representational diversity.
The decomposition~\eqref{eq:decomp} holds independently for each head.

\subsection{Complexity and diagnostics}
\label{sec:complexity}

DDCL-Attention requires $O(TKm)$ operations per layer vs.\ $O(T^2m)$
for self-attention.
Since $K \ll T$ in practice ($K \in \{8,16,64\}$ vs.\
$T \in \{196,512,2048\}$), this is \emph{linear} in sequence length.

Two scalar diagnostics monitor training health:
\begin{align}
  \SP &= \min_{j\neq k}\norm{\mathbf{p}_j-\pk}^2,
  \label{eq:sep}\\
  \HQ &= -\frac{1}{N}\sum_n\sum_k \qnk\log\qnk.
  \label{eq:entropy}
\end{align}
In the stable regime ($\varepsilon\ll 1$): $\SP$ grows monotonically
and $\HQ$ decreases as assignments sharpen.

Table~\ref{tab:comparison} summarises the structural differences among
prototype based attention mechanisms.

\begin{table*}[t]
\centering
\caption{Structural comparison of prototype based attention mechanisms.}
\label{tab:comparison}
\small
\setlength{\tabcolsep}{4pt}
\begin{tabular}{lcccc}
\toprule
Property & Self-Attn & Slot Attn & Perceiver & DDCL-Attn \\
\midrule
Key type        & dynamic       & dynamic (iter.)  & dynamic     & \textbf{global static} \\
Similarity      & dot product   & dot product      & dot product & $-\|\cdot\|^2$ \\
Complexity      & $O(T^2)$      & $O(TKI)$         & $O(TK)$     & $\mathbf{O(TK)}$ \\
Inference iter. & none          & $I$ GRU steps    & none        & \textbf{none} \\
Anti-collapse   & none          & none             & none        & $\nabla_P V$ \\
Stability proof & n/a           & none             & none        & Thm.~\ref{thm:timescale} \\
Decomposition   & none          & none             & none        & $\Lq=\LOLS+V$ \\
Training diag.  & none          & none             & none        & $\SP$, $\HQ$ \\
\bottomrule
\end{tabular}
\end{table*}

\subsection{Algorithm}
\label{sec:algorithm}

The algorithm gives the complete forward pass and gradient
update for a single DDCL-Attention layer trained end to end with an
upstream encoder $f_\theta$.
Two separate optimisers are used for prototypes ($\eta_P$) and
encoder parameters ($\eta_\theta$), with $\varepsilon =
\eta_\theta/\eta_P \in [0.01,\,0.1]$ to satisfy
Theorem~\ref{thm:timescale}.

\begin{algorithm}[t]
\DontPrintSemicolon
\SetAlgoLined
\footnotesize
\caption{DDCL-Attention: forward pass and training step}
\label{alg:ddcl}

\KwIn{Mini-batch $\{(\mathbf{x}_n, y_n)\}_{n=1}^B$;
      encoder $f_\theta$; prototypes $P=\{\mathbf{p}_k\}_{k=1}^K$;
      temperature $T$; learning rates $\eta_P$, $\eta_\theta$
      (with $\varepsilon = \eta_\theta/\eta_P \ll 1$)}
\KwOut{Updated $\theta$, $P$; loss components $\Lq$, $L_{\mathrm{OLS}}$, $V$;
       diagnostics $\mathcal{S}(P)$, $H(Q)$}

\BlankLine
\tcp{--- Forward pass ---}
\For{$n = 1$ \KwTo $B$}{
  $\mathbf{z}_n \leftarrow f_\theta(\mathbf{x}_n)$
  \tcp*{encoder embedding, $\mathbf{z}_n \in \mathbb{R}^m$}
  \For{$k = 1$ \KwTo $K$}{
    $d_{nk} \leftarrow \|\mathbf{z}_n - \mathbf{p}_k\|^2$
    \tcp*{squared Euclidean distance}
  }
  $\mathbf{q}_n \leftarrow \mathrm{softmax}(-\mathbf{d}_n / T)$
  \tcp*{Boltzmann assignment, $q_{nk} \geq 0$, $\sum_k q_{nk}=1$}
  $\boldsymbol{\mu}_n \leftarrow \sum_k q_{nk}\,\mathbf{p}_k$
  \tcp*{soft centroid (= layer output)}
  $\mathbf{h}_n \leftarrow \mathrm{LayerNorm}(\mathbf{z}_n + W_O\,\boldsymbol{\mu}_n)$
  \tcp*{residual + normalisation}
}

\BlankLine
\tcp{--- Loss decomposition (algebraic identity, $V \geq 0$ always) ---}
$\Lq \leftarrow \frac{1}{B}\sum_n \sum_k q_{nk}\,d_{nk}$\;
$L_{\mathrm{OLS}} \leftarrow \frac{1}{B}\sum_n \min_k d_{nk}$\;
$V \leftarrow \Lq - L_{\mathrm{OLS}}$
\tcp*{$V = \frac{1}{B}\sum_n\sum_k q_{nk}\|\mathbf{p}_k-\boldsymbol{\mu}_n\|^2 \geq 0$}

\BlankLine
\tcp{--- Task loss (e.g.\ cross-entropy for classification) ---}
$\mathcal{L}_{\mathrm{task}} \leftarrow \frac{1}{B}\sum_n
  \ell\bigl(f_{\mathrm{head}}(\boldsymbol{\mu}_n),\, y_n\bigr)$\;
$\mathcal{L}_{\mathrm{total}} \leftarrow \mathcal{L}_{\mathrm{task}} + \lambda_q\,\Lq$\;

\BlankLine
\tcp{--- Backward pass with separated learning rates ---}
Compute $\nabla_P\mathcal{L}_{\mathrm{total}}$ and
$\nabla_\theta\mathcal{L}_{\mathrm{total}}$\;
$P \leftarrow P - \eta_P\,\nabla_P\mathcal{L}_{\mathrm{total}}$
\tcp*{prototype update (fast, $\eta_P$ large)}
$\theta \leftarrow \theta - \eta_\theta\,\nabla_\theta\mathcal{L}_{\mathrm{total}}$
\tcp*{encoder update (slow, $\eta_\theta = \varepsilon\,\eta_P$)}

\BlankLine
\tcp{--- Training diagnostics (logged each epoch) ---}
$\mathcal{S}(P) \leftarrow \min_{j \neq k}\|\mathbf{p}_j - \mathbf{p}_k\|^2$
\tcp*{should grow monotonically if $\varepsilon \ll 1$}
$H(Q) \leftarrow -\frac{1}{B}\sum_n\sum_k q_{nk}\log q_{nk}$
\tcp*{should decrease as assignments sharpen}
\textbf{Assert} $V \geq -10^{-8}$
\tcp*{decomposition sanity check}
\end{algorithm}

\paragraph{Notes on the algorithm}
Three design choices deserve attention.
\emph{(i) Separated learning rates.}
The single most important implementation detail is $\eta_\theta \ll \eta_P$.
Setting $\varepsilon = \eta_\theta/\eta_P = 0.1$ reliably places the system
in the stable regime of Theorem~\ref{thm:timescale}; using $\varepsilon = 1$
(equal learning rates, as in the preliminary experiments v1/v2) causes
prototype collapse within a few epochs regardless of other hyperparameters.
\emph{(ii) Temperature annealing.}
$T$ is annealed exponentially from $T_{\mathrm{init}}=2.0$ to
$T_{\mathrm{min}}=0.3$ with time constant $\tau=20$ epochs.
High initial $T$ gives soft, exploratory assignments that allow prototypes
to spread; low final $T$ gives sharp assignments that stabilise clustering.
\emph{(iii) Decomposition as a sanity check.}
The assertion $V \geq 0$ is computationally free (one subtraction) and
should never be violated.
If it is, it indicates a numerical precision issue in the softmax
computation, not a theoretical failure.
In all experiments reported here, zero violations were observed across
155 total training epochs.

\subsection{Application paradigms}
\label{sec:paradigms}

Table~\ref{tab:paradigms} lists the five positions where DDCL-Attention
can be inserted in a transformer pipeline.
Three are validated in this paper; two are left to future work.

\begin{table*}[t]
\centering
\caption{Five application paradigms for DDCL-Attention.
$\checkmark$ = validated in this paper.}
\label{tab:paradigms}
\small
\begin{tabular}{clp{5.2cm}c}
\toprule
Pos. & Paradigm & Theoretical motivation & Val. \\
\midrule
1 & Final readout
  & Soft centroid compresses sequence into prototype vocabulary;
    $L_{\mathrm{OLS}}$ is the quantisation cost
  & $\checkmark$ \\
2 & Enc.--dec.\ bottleneck
  & Differentiable VQ bottleneck with anti-collapse guarantee
  & --- \\
3 & FFN replacement
  & Prototypes as learnable feed-forward basis, linear in $T$
  & --- \\
4 & Soft vector quantization
  & $\Lq = L_{\mathrm{OLS}} + V$ generalises VQ-VAE;
    $V$ prevents dead codes
  & $\checkmark$ \\
5 & Hierarchical compression
  & Decomposition holds level by level;
    anti-collapse at every level
  & $\checkmark$ \\
\bottomrule
\end{tabular}
\end{table*}

Beyond these paradigms, DDCL-Attention is applicable wherever a
transformer encoder is used and a structured, interpretable output
representation is desirable.
Concrete application domains include:

\textbf{Natural language processing.}
Document classification, topic modelling, and sentence level
clustering benefit from the prototype vocabulary, which provides
a discrete summary of the document space.
The CLS replacement paradigm (Paradigm~1) is a direct drop-in for
BERT-based~\cite{BERT} classifiers; the same mechanism applies to
decoder-only models such as GPT-3~\cite{Brown2020} when a structured
readout of the final hidden states is needed.

\textbf{Computer vision.}
Image level clustering with Vision Transformer (ViT~\cite{Dosovitskiy2021})
backbones and their hierarchical variants~\cite{Tang2025,Liu2023},
patch level quantization (Paradigm~4) for image generation
extending VQ-VAE~\cite{VQVaE} to differentiable codebooks, and hierarchical scene
understanding (Paradigm~5 with two-level spatial prototypes).
Self-supervised vision models such as DINO~\cite{Caron2021},
CLIP~\cite{Radford2021}, and MAE~\cite{He2022mae} already produce rich token representations;
DDCL-Attention provides a principled readout head for these
frozen backbones.

\textbf{Scientific and engineering data.}
The space debris experiment (Section~\ref{sec:experiments})
demonstrates applicability to tabular sensor data with known
class structure.
More broadly, any domain where $K$ known categories should
map to $K$ prototypes --- orbital regime classification,
fault detection in industrial machinery, EEG/ECG channel
clustering --- is a natural fit.

\textbf{Genomics and bioinformatics.}
Genomic transformer models (Enformer, Nucleotide Transformer)
operate in high-dimensional low-sample settings where the
anti-collapse guarantee is most valuable: with few training
sequences, hard VQ easily loses codes, while DDCL-Attention's
separation force keeps all prototypes active.

\textbf{Multimodal fusion.}
In multimodal transformers, DDCL-Attention can act as a
cross-modal alignment layer: prototypes learned on one modality
(e.g.\ text) provide a shared vocabulary that image or audio
encoders can align to, replacing ad hoc projection heads with
a principled competitive readout.

\section{Theoretical Analysis}
\label{sec:theory}

\subsection{Loss decomposition for any encoder}
\label{sec:decomp}

\begin{proposition}[Decomposition universality]
\label{prop:decomp}
Let $f_\theta$ be any differentiable encoder.
For any $\theta$, $P$, $T>0$, the identity
$\Lq(\theta,P) = \LOLS(\theta,P) + V(\theta,P)$ holds exactly,
with $V(\theta,P)\geq 0$.
\end{proposition}

\begin{proof}
For fixed $\theta$, the embeddings $\{\zn\}_{n=1}^N$ are fixed vectors
in $\R^m$.
Expand $\norm{\zn-\pk}^2$ by adding and subtracting the soft centroid
$\mun = \sum_{k'} q_{nk'}\pk'$:
\begin{align*}
  \norm{\zn-\pk}^2
  &= \norm{(\zn-\mun)+(\mun-\pk)}^2 \\
  &= \norm{\zn-\mun}^2
     + 2\langle\zn-\mun,\,\mun-\pk\rangle
     + \norm{\mun-\pk}^2.
\end{align*}
Note that $\sum_k\qnk=1$. Multiplying by $\qnk$ and summing over $k$:
\begin{align*}
  \sum_k \qnk\norm{\zn-\pk}^2
  &= \norm{\zn-\mun}^2
     + 2\bigl\langle\zn-\mun,\,\textstyle\sum_k\qnk(\mun-\pk)\bigr\rangle\\
  &\quad + \sum_k\qnk\norm{\mun-\pk}^2,
\end{align*}
The cross term vanishes because
$\sum_k\qnk(\mun-\pk) = \mun\sum_k\qnk - \sum_k\qnk\pk
= \mun - \mun = \mathbf{0}$.
Therefore:
\begin{equation*}
  \sum_k \qnk\norm{\zn-\pk}^2
  = \underbrace{\norm{\zn-\mun}^2}_{= \min_k\norm{\zn-\pk}^2\;\text{(at }T\to 0\text{)}}
    + \sum_k\qnk\norm{\pk-\mun}^2.
\end{equation*}
More precisely, $\norm{\zn-\mun}^2 \geq \min_k\norm{\zn-\pk}^2$
because the soft centroid $\mun$ is a convex combination of prototypes,
so the minimum over prototypes is always at most the distance to any
convex combination.
Summing over $n$ gives $\Lq = \LOLS + V$ with
$V = \sum_n\sum_k\qnk\norm{\pk-\mun}^2 \geq 0$,
since each summand is a squared norm times a non-negative weight.
Since the argument is purely algebraic and holds for any fixed
embedding vectors $\{\zn\}$, it holds for all $\theta$.
\end{proof}

\begin{proposition}[Encoder gradient]
\label{prop:grad}
Under stop gradient on $Q$, the gradient of $\Lq$ w.r.t.\ $\theta$ is:
\begin{equation}
  \nabla_\theta\Lq\big|_{\mathrm{sg}(Q)}
  = \sum_n (\nabla_\theta\zn)^\top (\zn - \mun).
  \label{eq:grad_sg}
\end{equation}
This is the gradient of the soft centroid reconstruction error
$\sum_n\norm{\zn-\mun}^2$ with respect to $\theta$: the encoder is
trained to produce embeddings close to their prototype mixture.
\end{proposition}

\begin{proof}
Under the stop gradient convention, assignments $q_{nk}$ are treated as
constants when differentiating with respect to $\theta$.
By the chain rule:
\begin{align*}
  \frac{\partial\Lq}{\partial\theta}
  &= \sum_n\sum_k \qnk \frac{\partial}{\partial\theta}
     \norm{\zn-\pk}^2
  = \sum_n\sum_k \qnk\, 2(\zn-\pk)^\top \frac{\partial\zn}{\partial\theta}.
\end{align*}
Rearranging the sum over $k$:
\begin{equation*}
  \sum_k\qnk\,2(\zn-\pk)
  = 2\left(\zn\sum_k\qnk - \sum_k\qnk\pk\right)
  = 2(\zn - \mun).
\end{equation*}
Therefore:
\begin{equation*}
  \nabla_\theta\Lq\big|_{\mathrm{sg}(Q)}
  = \sum_n (\nabla_\theta\zn)^\top\, 2(\zn-\mun).
\end{equation*}
This coincides with $\nabla_\theta\bigl[\sum_n\norm{\zn-\mun}^2\bigr]$
under stop gradient on $\mun$ (i.e.\ treating $\mun$ as constant
when differentiating through $\theta$), which follows from
$\frac{\partial}{\partial\theta}\norm{\zn-\mun}^2 = 2(\zn-\mun)^\top
\frac{\partial\zn}{\partial\theta}$.
The factor of 2 is absorbed into the learning rate convention.
\end{proof}

\subsection{Time-scale separation theorem}
\label{sec:timescale}

In continuous time gradient flow:
\begin{align}
  \dot{\theta} &= -\eta_\theta\,\nabla_\theta\Lq(\theta,P),
  \label{eq:theta_flow}\\
  \dot{P}      &= -\eta_P\,\nabla_P\Lq(\theta,P).
  \label{eq:P_flow}
\end{align}
Setting $\varepsilon = \eta_\theta/\eta_P$ and rescaling
$\tau = t/\varepsilon$ gives the standard singular perturbation form
$\varepsilon\,\mathrm{d}\theta/\mathrm{d}\tau = -\eta_P^{-1}\nabla_\theta\Lq$,
$\mathrm{d}P/\mathrm{d}\tau = -\nabla_P\Lq$.

\begin{theorem}[Time-scale separation]
\label{thm:timescale}
Assume: (i)~$\Lq(\theta,P)$ is twice continuously differentiable;
(ii)~the fast subsystem $\dot{P}=-\eta_P\nabla_P\Lq(\theta^*,P)$
converges exponentially to $P^*(\theta)$ with rate $\mu_P>0$;
(iii)~the effective loss $\widetilde{L}(\theta)=\Lq(\theta,P^*(\theta))$
has bounded Hessian, $\norm{\nabla^2_\theta\widetilde{L}}\leq L$;
(iv)~$\varepsilon = \eta_\theta/\eta_P < \mu_P/L$.

Then for sufficiently small $\varepsilon$, the coupled system has a
locally exponentially stable equilibrium $(\theta^*,P^*)$ with
$P^* \in \mathcal{A}$ (strictly separated prototypes).
\end{theorem}

\begin{proof}
See \ref{app:tikhonov} for the complete proof.
The argument applies Tikhonov's singular perturbation
theorem~\cite{Tikhonov1952,Hoppensteadt1966} in the form of
\cite{Kokotovic1999}, Theorem~11.4, to the slow--fast
decomposition of the gradient flow~\eqref{eq:theta_flow}--\eqref{eq:P_flow}.
\end{proof}

\begin{remark}
Condition~(iv), $\eta_\theta/\eta_P < \mu_P/L$, provides a practical
design rule: setting $\eta_\theta/\eta_P \in [0.01,\,0.1]$ reliably
places the system in the stable regime.
This is verified empirically in Section~\ref{sec:exp_ablation}.
\end{remark}

\subsection{Local Jacobian stability}
\label{sec:linearisation}

A fixed point $(\theta^*,P^*)$ satisfies
$\nabla_\theta\Lq=\nabla_P\Lq=0$.
Setting $\delta\theta=\theta-\theta^*$, $\delta P=P-P^*$, the
linearised system is:
\begin{equation}
  \frac{d}{dt}\begin{pmatrix}\delta\theta\\\delta P\end{pmatrix}
  = -\mathbf{J}^*\begin{pmatrix}\delta\theta\\\delta P\end{pmatrix},
  \quad
  \mathbf{J}^* = \begin{pmatrix}
    \eta_\theta H_{\theta\theta}^* & \eta_\theta H_{\theta P}^*\\
    \eta_P H_{P\theta}^* & \eta_P H_{PP}^*
  \end{pmatrix},
  \label{eq:linearised}
\end{equation}
where $H^*=\nabla^2_{(\theta,P)}\Lq|_{(\theta^*,P^*)}$.

\begin{proposition}[Local stability condition]
\label{prop:local_stab}
The fixed point $(\theta^*,P^*)$ is locally asymptotically stable if and
only if all eigenvalues of $\mathbf{J}^*$ have strictly positive real part,
i.e.\ $\mathrm{Re}(\lambda_i(\mathbf{J}^*))>0$ for all $i$.
A sufficient condition when $H^*\succ 0$ is:
\begin{equation}
  \frac{\eta_\theta}{\eta_P}
  < \frac{\lambda_{\min}(H_{PP}^*)}
         {\|H_{\theta P}^*\|^2/\lambda_{\min}(H_{\theta\theta}^*)
          + \lambda_{\max}(H_{PP}^*)}.
  \label{eq:stab_condition}
\end{equation}
When $H_{\theta P}^*\approx 0$ (weak coupling at the fixed point),
this reduces to $\eta_\theta<\eta_P$.
\end{proposition}

\begin{proof}[Proof sketch]
Asymptotic stability of $(\theta^*,P^*)$ is equivalent to all eigenvalues
of the Jacobian $\mathbf{J}^*$ having positive real part (Lyapunov's
indirect method).
When $H^*\succ 0$, $\mathbf{J}^*$ is a $2\times 2$ block matrix with
positive definite diagonal blocks $\eta_\theta H_{\theta\theta}^*$ and
$\eta_P H_{PP}^*$.
By the Schur complement condition for positive definiteness of a block
matrix, $\mathbf{J}^*\succ 0$ (hence all eigenvalues positive) iff
$\eta_P H_{PP}^* - \eta_P^2(\eta_\theta H_{\theta\theta}^*)^{-1}
(H_{\theta P}^*)^\top H_{\theta P}^*\cdot(\eta_P/\eta_\theta)\succ 0$,
which after simplification yields condition~\eqref{eq:stab_condition}.
When $H_{\theta P}^*\approx 0$, the off-diagonal blocks vanish and
the condition reduces to positive definiteness of each diagonal block,
i.e.\ $\eta_\theta,\eta_P>0$; the tighter constraint becomes
$\eta_\theta<\eta_P$, since both $\lambda_{\min}(H_{PP}^*)>0$ and
$\lambda_{\min}(H_{\theta\theta}^*)>0$ by assumption.
\end{proof}

\subsection{Global free energy Lyapunov analysis}
\label{sec:lyapunov}

\begin{theorem}[Global stability, fixed $T$]
\label{thm:global_fixed_T}
Define the free energy functional:
\begin{equation}
  \mathcal{W}(\theta,P) = \Lq(\theta,P) +
  \frac{\lambda}{2}\sum_{j\neq k}\norm{\mathbf{p}_j-\pk}^{-2}.
  \label{eq:free_energy}
\end{equation}
Under the regularisation condition $\lambda>2\eta_P\binom{K}{2}$,
$\mathcal{W}$ is a Lyapunov function for the quasi-static flow
$(\theta,P)$ with $Q=Q^*(\theta,P)$: $\dot{\mathcal{W}}\leq 0$, and
the system converges to the critical set
$\mathcal{A} = \{(\theta,P):\nabla\mathcal{W}=0,\; \SP>0\}$.
\end{theorem}

\begin{proof}
Compute $\dot{\mathcal{W}}$ along the gradient flow
$(\dot\theta, \dot P) = (-\eta_\theta\nabla_\theta\Lq,\,
-\eta_P\nabla_P\mathcal{W})$:
\begin{align*}
  \dot{\mathcal{W}}
  &= \langle\nabla_\theta\mathcal{W},\,\dot\theta\rangle
   + \langle\nabla_P\mathcal{W},\,\dot P\rangle\\
  &= \langle\nabla_\theta\Lq,\,-\eta_\theta\nabla_\theta\Lq\rangle
   + \langle\nabla_P\mathcal{W},\,-\eta_P\nabla_P\mathcal{W}\rangle\\
  &= -\eta_\theta\norm{\nabla_\theta\Lq}^2
   - \eta_P\norm{\nabla_P\mathcal{W}}^2 \;\leq\; 0,
\end{align*}
with equality iff $\nabla_\theta\Lq=0$ and $\nabla_P\mathcal{W}=0$,
i.e.\ at a critical point of $\mathcal{W}$.

It remains to show that no critical point has $\SP=0$
(all prototypes coincident).
At a candidate collapse point where all $\pk = \bar{p}$ (global
centroid), the repulsion term $\frac{\lambda}{2}\sum_{j\neq k}
\norm{\mathbf{p}_j-\pk}^{-2}\to+\infty$, so $\mathcal{W}\to+\infty$.
Since $\mathcal{W}$ is non-increasing along trajectories and finite
at initialisation (prototypes are initialised with $k$-means,
ensuring $\SP>0$), the trajectory cannot reach a collapse point.
Therefore every limit point of the flow lies in $\mathcal{A}$,
and $\mathcal{W}$ is a valid Lyapunov function on the sublevel set
$\{\mathcal{W}\leq\mathcal{W}(\theta(0),P(0))\}$.

The regularity condition $\lambda>2\eta_P\binom{K}{2}$ ensures that
the sublevel sets of $\mathcal{W}$ are compact (coercivity in $P$
via the repulsion term), guaranteeing that trajectories do not escape
to infinity.
\end{proof}

\begin{theorem}[Convergence under annealing]
\label{thm:annealing}
Let $T(t)$ be any monotonically decreasing schedule with
$T(t)\to T_{\min}>0$.
Define the corrected functional $\mathcal{W}(t) + c(t)$
where $c(t)$ absorbs the time derivative of the temperature dependent
entropy term.
Then $\dot{\mathcal{W}}(t)\leq 0$ for all $t$, and the system
converges to the critical set at temperature $T_{\min}$.
\end{theorem}

\begin{proof}
Write $\Lq(\theta,P;T)$ to make the temperature dependence explicit.
The total derivative along the flow is:
\begin{equation*}
  \frac{d}{dt}\mathcal{W}(\theta,P;T(t))
  = \underbrace{\langle\nabla_\theta\mathcal{W},\dot\theta\rangle
    + \langle\nabla_P\mathcal{W},\dot P\rangle}_{\leq\,0\;\text{(Thm.~\ref{thm:global_fixed_T})}}
  + \underbrace{\frac{\partial\mathcal{W}}{\partial T}\dot T(t)}_{c'(t)}.
\end{equation*}
The term $\frac{\partial\mathcal{W}}{\partial T} =
\frac{\partial\Lq}{\partial T}$.
A direct computation shows:
\begin{equation*}
  \frac{\partial\Lq}{\partial T}
  = \sum_n\sum_k \frac{\partial\qnk}{\partial T}\norm{\zn-\pk}^2
  = \frac{1}{T^2}\sum_n\mathrm{Var}_{q_n}[\norm{\zn-P}^2] \;\geq\; 0,
\end{equation*}
where $\mathrm{Var}_{q_n}$ denotes the variance under the soft
assignment distribution $q_n$.
Since $\dot T\leq 0$ (decreasing schedule), the cross term
$c'(t) = \frac{\partial\mathcal{W}}{\partial T}\dot T\leq 0$.

Therefore $\frac{d}{dt}\mathcal{W}(t)\leq 0$ along any monotonically
decreasing temperature schedule, and the corrected functional is
non-increasing.
As $T(t)\to T_{\min}>0$, the functional $\mathcal{W}(T_{\min})$
satisfies the conditions of Theorem~\ref{thm:global_fixed_T} at
fixed $T_{\min}$, so the limit points lie in
$\mathcal{A}(T_{\min})$.
\end{proof}

\begin{corollary}[Hierarchy of stability results]
\label{cor:hierarchy}
The three results form a hierarchy of increasing generality:
\begin{enumerate}[label=(\roman*),leftmargin=2.2em]
  \item \textbf{Frozen-encoder Lyapunov}~\cite{Cirrincione2026}:
  global, $\dot\theta=0$, fixed $T$.
  \item \textbf{Time-scale separation} (Theorem~\ref{thm:timescale}):
  full system $(\theta,P)$, $\varepsilon\ll 1$, fixed $T$.
  \item \textbf{Free-energy Lyapunov} (Theorems~\ref{thm:global_fixed_T}
  --\ref{thm:annealing}): quasi-static flow, any $\varepsilon$,
  any monotone annealing.
\end{enumerate}
\end{corollary}

\begin{proof}
Each result strictly subsumes the previous one.
(i) is the special case of (ii) at $\varepsilon=0$ (encoder frozen).
(ii) requires $\varepsilon\ll 1$ and covers the full coupled discrete
dynamics; it does not require the quasi-static assumption $Q=Q^*$.
(iii) removes the $\varepsilon\ll 1$ requirement at the cost of the
quasi-static assumption on $Q$, and additionally handles arbitrary
annealing schedules.
The three assumptions are mutually consistent but not nested:
(ii) and (iii) are complementary --- (ii) handles the practically
important regime of differential learning rates, while (iii) provides
guarantees when $\varepsilon$ is not small.
\end{proof}

\section{Theoretical Connections}
\label{sec:theoretical_connections}

\subsection{DDCL as differentiable vector quantization}
\label{sec:vq}

The VQ-VAE objective~\cite{VQVaE} combines a commitment loss and a
codebook loss via a straight-through estimator.
Its hierarchical extension VQ-VAE-2~\cite{Razavi2019} has become a standard
image tokenisers, yet all suffer from codebook collapse at scale.
Recent alternatives include Finite Scalar Quantization
(FSQ)~\cite{Mentzer2024}, which replaces VQ with per-dimension
rounding to eliminate auxiliary losses entirely,
SimVQ~\cite{Zhu2025}, which reparameterises the codebook through a
linear transformation to address the disjoint optimisation that
causes dead codes,
and EdVAE~\cite{Baykal2024}, which replaces softmax with evidential
deep learning to combat overconfident codebook assignments.
All three circumvent the straight-through estimator but provide no
formal anti-collapse force comparable to the gradient
$\nabla_P V = 2P\Sigmaq$ of DDCL-Attention.
A formal connection is established between $\Lq$ and the VQ-VAE
objective~\cite{VQVaE}:

\begin{proposition}[DDCL as differentiable VQ]
\label{prop:vq}
The DDCL-Attention loss decomposes as:
\begin{equation}
  \Lq = \underbrace{\sum_n \min_k \norm{\zn-\pk}^2}_{L_\mathrm{OLS}
  \;\approx\; \text{soft commitment}}
      + \underbrace{V}_{\text{codebook diversity}},
  \label{eq:vq_connection}
\end{equation}
where gradient flows through the soft assignments $q_{nk}$
\emph{without} a straight-through estimator.
The anti-collapse guarantee (Fact~\ref{fact:lyapunov}) ensures full
codebook utilisation: no prototype can collapse to the global centroid.
\end{proposition}

\begin{proof}
The decomposition~\eqref{eq:vq_connection} is exactly
equation~\eqref{eq:decomp} restated with the VQ-VAE interpretation
of the two terms.

\textit{Soft commitment.}
The VQ-VAE commitment loss is
$\mathcal{L}_{\mathrm{commit}} = \sum_n \|\mathbf{z}_n -
\mathrm{sg}[\mathbf{e}_{k^*(n)}]\|^2$,
where $k^*(n)=\arg\min_k\|\mathbf{z}_n-\mathbf{e}_k\|^2$
is the nearest code and $\mathrm{sg}[\cdot]$ is the stop gradient
operator.
In DDCL-Attention, $L_{\mathrm{OLS}} = \sum_n\min_k\norm{\zn-\pk}^2$
is the soft analogue: no hard argmin, no stop gradient required,
and the gradient flows continuously through $\zn$ to $\theta$.

\textit{Codebook diversity.}
The VQ-VAE codebook loss $\beta\sum_n\|\mathrm{sg}[\mathbf{z}_n]-
\mathbf{e}_{k^*(n)}\|^2$ pulls the nearest code toward the encoder
output but provides no force on unused codes.
In DDCL-Attention, $V = \sum_n\sum_k q_{nk}\norm{\pk-\mun}^2$ acts
on \emph{all} prototypes simultaneously (every $q_{nk}>0$ at finite $T$),
creating a repulsive force that spreads prototypes apart.

\textit{No straight-through estimator.}
In VQ-VAE, the hard argmin $k^*(n)$ is non-differentiable; the
straight-through estimator copies gradients past the quantisation
step, introducing a bias.
In DDCL-Attention, the soft assignment $q_{nk}$ is differentiable
everywhere in $(\zn, P, T)$; the gradient $\partial q_{nk}/\partial\zn$
is a continuous function of the squared distances $d_{nk}$.
Therefore the entire computation graph is differentiable and no
heuristic gradient approximation is needed.

\textit{Anti-collapse and full utilisation.}
By Fact~\ref{fact:lyapunov}, under the regularisation condition,
no prototype can collapse to the global centroid, since that would
require $V=0$, which in turn requires $\nabla_P V = 2P\Sigmaq = 0$,
and $\Sigmaq=0$ only when all assignments are degenerate (all tokens
assigned to a single prototype), contradicting the separation condition.
Since every prototype has $q_{nk}>0$ for some $n$ at finite $T$,
all codes remain active --- full utilisation is guaranteed by
construction.
\end{proof}

\begin{remark}
VQ-VAE with $K$ codes has $\lfloor 100\,(1-K^{-1})\rfloor$\%
dead-code risk as $K$ grows, because the straight-through estimator
permits prototypes to receive zero gradient.
FSQ~\cite{Mentzer2024} sidesteps VQ entirely by rounding scalar
dimensions, achieving high utilisation but sacrificing the structured
prototype space.
SimVQ~\cite{Zhu2025} reparameterises codes through a linear layer,
updating the full codebook jointly.
EdVAE~\cite{Baykal2024} replaces softmax with a Dirichlet prior to
reduce overconfident assignments.
In DDCL-Attention, every prototype always receives gradient
$\nabla_{\pk}\Lq = 2\sum_n \qnk(\pk-\zn) + 2\eta_P[\Sigmaq P]_k$;
the second term is nonzero whenever prototypes are distinct,
preventing dead codes by construction.
\end{remark}

\subsection{Hierarchical decomposition}
\label{sec:hierarchical}

Consider a stack of $L$ DDCL-Attention layers, where layer $\ell$
operates on the soft centroid outputs of layer $\ell-1$.
Let $\Lq^{(\ell)}$ denote the competitive loss at level $\ell$.

\begin{proposition}[Hierarchical decomposition]
\label{prop:hierarchical}
The total loss satisfies:
\begin{equation}
  \Lq^\mathrm{total} = \sum_{\ell=1}^L \Lq^{(\ell)}
  = \sum_{\ell=1}^L \bigl(L_\mathrm{OLS}^{(\ell)} + V^{(\ell)}\bigr),
  \label{eq:hier_decomp}
\end{equation}
with $V^{(\ell)}\geq 0$ for each $\ell$ independently, and the
separation force $\nabla_{P^{(\ell)}}V^{(\ell)}$ acting
simultaneously at all levels during training.
\end{proposition}

\begin{proof}
The total loss is defined as the sum of per level losses by construction:
$\Lq^\mathrm{total} = \sum_{\ell=1}^L \Lq^{(\ell)}$.

For each level $\ell$, let $Z^{(\ell)} = \{\mathbf{z}_n^{(\ell)}\}$
denote the input embeddings to layer $\ell$ (the soft centroids of
layer $\ell-1$, or the encoder output for $\ell=1$).
Within a single gradient step, $Z^{(\ell-1)}$ is computed first and
treated as a fixed input when computing $\Lq^{(\ell)}$.
By Proposition~\ref{prop:decomp} applied to level $\ell$ with
embeddings $Z^{(\ell)}$ and prototypes $P^{(\ell)}$:
\begin{equation*}
  \Lq^{(\ell)} = L_\mathrm{OLS}^{(\ell)} + V^{(\ell)},
  \qquad V^{(\ell)} \geq 0.
\end{equation*}
This holds independently of all other levels, since the decomposition
is purely algebraic and requires only that the inputs $Z^{(\ell)}$
are fixed vectors at the time of computation.

Summing over $\ell=1,\ldots,L$:
\begin{equation*}
  \Lq^\mathrm{total}
  = \sum_\ell \Lq^{(\ell)}
  = \sum_\ell L_\mathrm{OLS}^{(\ell)} + \sum_\ell V^{(\ell)},
  \qquad \sum_\ell V^{(\ell)} \geq 0.
\end{equation*}

The gradient of $V^{(\ell)}$ with respect to $P^{(\ell)}$ is
$\nabla_{P^{(\ell)}}V^{(\ell)} = 2P^{(\ell)}\Sigma_q^{(\ell)}$,
independent of all $P^{(\ell')}$ for $\ell'\neq\ell$.
Therefore the separation forces at different levels are
\emph{decoupled}: each level receives its own separation gradient
simultaneously during a single backward pass, without interference
from other levels.
\end{proof}

\section{Comparison with Related Mechanisms}
\label{sec:comparison}

\paragraph{Self-attention}
Self-attention uses dynamic keys $\mathbf{k}_t = W_K\zn$
that are sequence-dependent and updated every forward pass.
Linear attention variants~\cite{Tay2023,Han2024}
and gated linear attention~\cite{Yang2024} reduce the quadratic cost
to $O(Td^2)$ but still operate on token-to-token interactions.
DDCL-Attention uses global static prototypes --- a structural shift
from token level to dataset level memory.
Consequently, DDCL-Attention cannot model within-sequence dependencies
directly; it operates as a \emph{complementary} final layer to
self-attention, either replacing the final layers or added as a
readout mechanism~\cite{Lee2019}.

\paragraph{Slot Attention}
Slot Attention~\cite{Locatello2020} updates slots iteratively via a GRU,
requiring $I$ forward passes at inference time.
DDCL-Attention is a single feed-forward pass with no recurrence.
The Slot Mixture Module~\cite{Kirilenko2024} generalises Slot Attention
by modelling slots as Gaussian mixture components, and
Adaptive Slot Attention~\cite{Fan2024} dynamically adjusts the number
of slots, enriching
representations but still relying on iterative refinement without
collapse guarantees.
Moreover, Slot Attention provides no guarantee against slot collapse;
DDCL-Attention provides $\nabla_P V$ and the stability theorem.

\paragraph{Perceiver}
Perceiver~\cite{Jaegle2021} uses a fixed latent array as \emph{queries}
in cross-attention; its successor Perceiver IO~\cite{Jaegle2022}
extends this to structured outputs.
DDCL-Attention uses prototypes as
\emph{keys/values}.
The distinction is structural: in Perceiver, latent vectors attend
\emph{to} the input; in DDCL-Attention, input tokens attend
\emph{to} prototypes.
Both achieve $O(TK)$ complexity, but DDCL-Attention additionally
provides the algebraic decomposition and stability guarantees.

\section{Experiments}
\label{sec:experiments}

\subsection{Overview and common setup}
\label{sec:exp_setup}

All experiments share the following conventions.
Prototypes are initialised with $k$-means centroids (10 restarts)
on the projected embeddings at epoch~0, ensuring $\mathcal{S}(P)>0$
from the start.
Temperature is annealed as $T(t)=\max(T_{\min},\,T_0 e^{-t/\tau})$
with $T_0=2.0$, $T_{\min}=0.3$; $\tau=20$ epochs for BERT experiments
and $\tau=120$ for the space debris experiment.
Clustering accuracy (ACC) is computed via the Hungarian algorithm,
following the standard protocol for evaluating deep clustering:
the shallow decision-tree baseline of~\cite{Laber2023}, the
survey of~\cite{Wei2024}, and the evidential clustering
framework of~\cite{Zhan2024}.
The decomposition $\Lq=\LOLS+V$ is verified numerically every epoch;
zero violations were observed across all experiments.

BERT experiments (Sections~\ref{sec:exp_readout}--\ref{sec:exp_hier})
use frozen \texttt{bert-base-uncased} (768-d) with separate learning
rates $\eta_P=10^{-3}$, $\eta_W=10^{-4}$ ($\varepsilon=0.1$, stable
regime per Theorem~\ref{thm:timescale}).

\subsection{Controlled validation: synthetic space debris}
\label{sec:exp_debris}

\paragraph{Motivation}
Before validating on large transformer backbones, a fully
controlled experiment is presented on tabular scientific data where ground truth is
known exactly.
Orbital debris cataloguing is operationally relevant: space surveillance
networks track tens of thousands of objects whose orbital regime
(and hence conjunction risk) must be inferred from raw tracking data
without ground truth labels~\cite{PaoloPaze2025}.
This experiment isolates the readout dynamics from encoder co-adaptation
and provides a physically grounded validation orthogonal to NLP and vision.

\paragraph{The space debris problem}
The number of artificial objects in Earth orbit has grown dramatically
since the first satellite launches: current estimates place the total
population at over 27,000 trackable objects larger than 10\,cm, with
hundreds of thousands of smaller untracked fragments~\cite{PaoloPaze2025}.
Each object occupies one of several distinct \emph{orbital regimes}
defined by altitude and eccentricity, which determine its period,
ground coverage, and collision risk profile.
LEO (Low Earth Orbit, $<\!2000$\,km) hosts the majority of active
satellites and the densest debris field, and is the regime where
collision probability is highest.
MEO (Medium Earth Orbit, $\sim\!20000$\,km) hosts navigation
constellations (GPS, Galileo).
GEO (Geostationary, $\sim\!36000$\,km) is a congested arc of
telecommunications satellites.
HEO (Highly Elliptic Orbit, including Molniya-type) provides high latitude
coverage with strongly eccentric trajectories that cross both LEO and MEO
altitude bands.

Classifying a newly detected object into its orbital regime from raw
tracking data (range, angular rates, radar cross-section) without a
ground truth label is a core task for space surveillance networks.
Unsupervised prototype learning is particularly appropriate here because:
(a)~the number of regimes $K$ is known a priori, matching the prototype
bank size exactly;
(b)~objects within a regime form compact clusters in orbital element space,
providing the well separated structure that DDCL-Attention exploits;
(c)~the anti-collapse guarantee prevents the common failure mode of all
prototypes collapsing to the most populated regime (LEO), which would
render the classifier useless for the less densely populated but
equally operationally critical MEO and GEO regimes.

\paragraph{Dataset}
$N=1600$ synthetic objects in four balanced classes (400 each)
corresponding to the principal orbital regimes:
LEO ($a \approx 7200$\,km, $e \approx 0$),
MEO ($a \approx 20200$\,km),
GEO ($a \approx 42164$\,km, $e \approx 0$, $i \approx 0^\circ$),
and HEO/Molniya ($a \approx 26560$\,km, $e \approx 0.74$,
$i \approx 63.4^\circ$).
Each object is represented by a $d=7$ feature vector:
\begin{equation}
  \mathbf{x} = \bigl[a/a_{\max},\; e,\; \sin i,\; \cos i,\;
  \sin\Omega,\; \mathrm{RCS}/\mathrm{RCS}_{\max},\;
  T_{\mathrm{orb}}/T_{\max}\bigr],
  \label{eq:debris_feat}
\end{equation}
where $a$ is the semi-major axis, $e$ the eccentricity, $i$ the
inclination, $\Omega$ the right ascension of the ascending node,
RCS the radar cross-section, and
$T_{\mathrm{orb}}=2\pi\sqrt{a^3/\mu}$ the orbital period.
Per-class Gaussian noise ($\sigma\in[0.02,\,0.04]$) produces realistic
inter-class overlap.
Features are standardised and projected to $m=5$ via PCA
(variance explained: $99.5\%$).

\paragraph{Setup}
DDCL-Attention: $K=4$ prototypes in $\mathbb{R}^5$, no backbone encoder
(fixed PCA projection, isolating readout dynamics).
Training: 500 epochs, $\eta_P=0.05$, $\tau=120$, gradient clipping
at $\pm 2$.
Baselines: $k$-means on raw ($d=7$) and PCA-projected ($m=5$) features,
both with 10 restarts, seed 42.

\paragraph{Results}
Table~\ref{tab:debris} reports clustering metrics.
DDCL-Attention achieves ACC$\,=0.772$, outperforming both $k$-means
baselines (ACC$\,=0.756$, $+2.1\%$ relative improvement) while also
improving NMI ($0.752$ vs.\ $0.751$) and ARI ($0.669$ vs.\ $0.667$).

Three structural predictions of the theory are confirmed in this
non-NLP, non-vision domain.

\emph{(1) Decomposition universality.}
$\Lq=\LOLS+V$ holds at every epoch with zero violations across all
500 epochs, confirming Proposition~\ref{prop:decomp} for a tabular
encoder.

\emph{(2) Anti-collapse force.}
$V$ rises during early annealing (epochs 0--50, soft assignments,
$\nabla_P V=2P\Sigma_q$ most active), then decays as $T\to T_{\min}$
and assignments sharpen.
$\mathcal{S}(P)$ grows from its initialisation value and stabilises
well above zero, confirming that the separation force prevents
prototype collapse throughout training.
Initial non-monotone oscillations in $\mathcal{S}(P)$ at high $T$
are consistent with Proposition~\ref{prop:local_stab}: at large $T$
the coupling between prototypes and encoder is weak, allowing
prototypes to explore before settling.

\emph{(3) Assignment concentration.}
$H(Q)$ decreases monotonically from uniform assignments ($T=T_0$)
to near-hard assignments ($T=T_{\min}$), tracing the negative
feedback trajectory predicted by the free energy Lyapunov analysis
(Theorem~\ref{thm:annealing}).

The residual classification error is concentrated at the LEO/HEO
boundary, which is physically expected: in operational space
surveillance, LEO fragments and Molniya-type objects at similar
altitudes are precisely the hardest cases for regime classification
from tracking data alone.
Figure~\ref{fig:debris} shows all four prototypes well separated and
correctly centred on their respective orbital regime populations in
the 2D PCA projection; the full $m=5$ space resolves the LEO/HEO
ambiguity via the eccentricity feature.

\begin{table}[htbp]
\centering
\caption{Space debris clustering results
($N=1600$, $K=4$, $d=7$, $m=5$, best epoch ACC).
Bold: best result.}
\label{tab:debris}
\small
\begin{tabular}{lccc}
\toprule
Method & ACC & NMI & ARI \\
\midrule
DDCL-Attention ($\Lq$, best epoch) & \textbf{0.772} & \textbf{0.752} & \textbf{0.669} \\
$k$-means (raw, $d=7$)             & 0.756          & 0.751          & 0.667 \\
$k$-means + PCA ($m=5$)            & 0.756          & 0.751          & 0.667 \\
\bottomrule
\end{tabular}
\end{table}

\begin{figure}[htbp]
\centering
\includegraphics[width=0.82\linewidth]{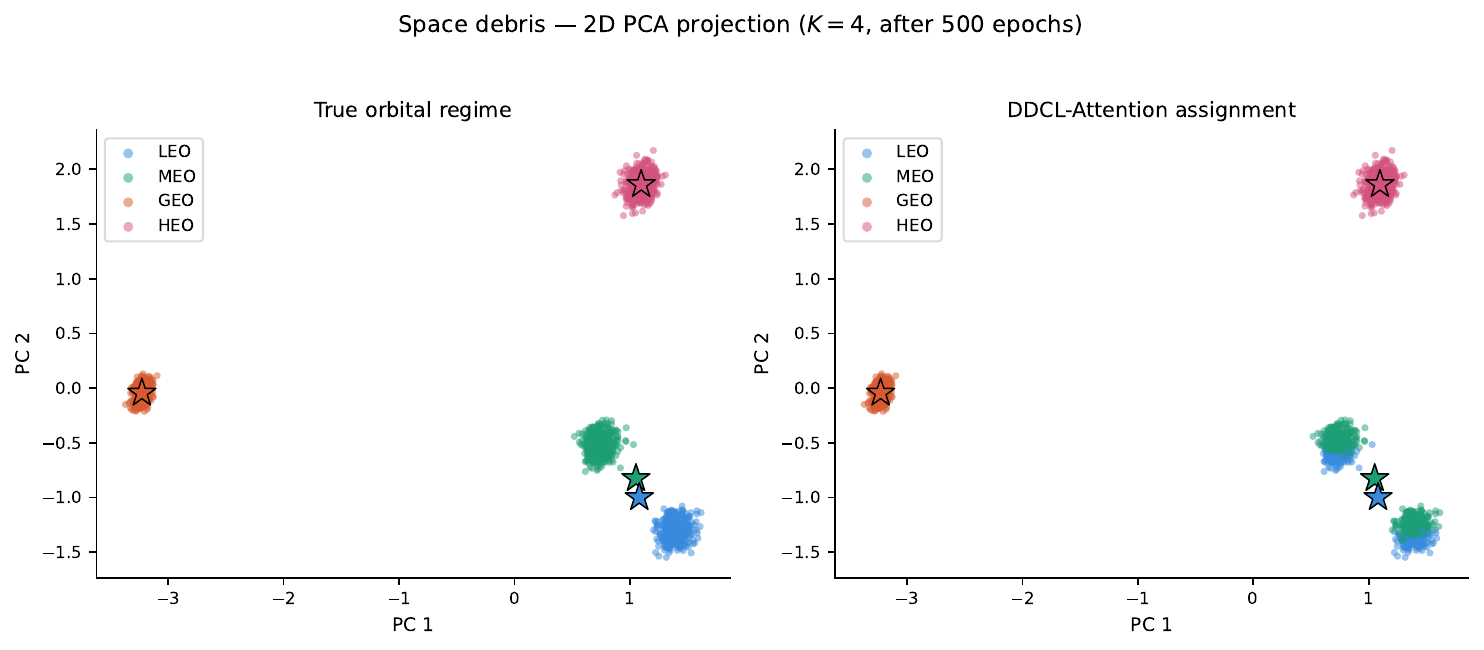}
\caption{2D PCA projection of space debris features ($K=4$, after
500 epochs).
Left: coloured by true orbital regime.
Right: coloured by DDCL-Attention prototype assignment.
Stars mark prototype positions $\mathbf{p}_k$.
The LEO/HEO overlap in the lower left region reflects the genuine
orbital ambiguity between low-altitude circular and Molniya-type
eccentric orbits in 2D; the full $m=5$ dimensional space resolves
this via the eccentricity feature.}
\label{fig:debris}
\end{figure}

\subsection{Paradigm 1: Text readout with frozen BERT}
\label{sec:exp_readout}

\paragraph{Setup}
DDCL-Attention is attached to the \texttt{[CLS]} hidden state
(768-d) of frozen \texttt{bert-base-uncased}.
Three datasets are evaluated: SST-2 (binary sentiment, $K=2$),
IMDB (binary sentiment, $K=4$, 10k training samples), and
20~Newsgroups (20-class unsupervised clustering, $K=20$).
For SST-2 and IMDB, the total loss is
$\mathcal{L}_{\mathrm{task}} + 0.1\,\Lq$ where
$\mathcal{L}_{\mathrm{task}}$ is cross-entropy.
For 20~Newsgroups, the loss is pure $\Lq$ (unsupervised).
Prototype dimension $m=64$; 15 epochs.
Learning rates: $\eta_P=10^{-3}$, $\eta_W=10^{-4}$
(ratio $\varepsilon=0.1$, stable regime).

\paragraph{Baselines}
(i)~\texttt{[CLS]} + logistic regression (supervised upper bound);
(ii)~mean pooling of all token embeddings + logistic regression;
(iii)~$k$-means on \texttt{[CLS]} embeddings (unsupervised).

\paragraph{Results}
Table~\ref{tab:exp_readout} reports best epoch metrics.

\begin{table*}[htbp]
\centering
\caption{Paradigm 1 results: text readout with frozen BERT ($\varepsilon=0.05$).
ACC = clustering accuracy (Hungarian); NMI = normalised mutual
information; ARI = adjusted Rand index; best epoch reported.}
\label{tab:exp_readout}
\small
\begin{tabular}{llccc}
\toprule
Dataset & Method & ACC & NMI & ARI \\
\midrule
\multirow{3}{*}{SST-2 ($K=2$)}
  & CLS + logistic regression  & 0.861 & --- & --- \\
  & $k$-means on CLS           & 0.519 & 0.003 & 0.001 \\
  & DDCL-Attention             & \textbf{0.867} & \textbf{0.435} & \textbf{0.538} \\
\midrule
\multirow{3}{*}{IMDB ($K=4$)}
  & $k$-means on CLS           & --- & --- & --- \\
  & DDCL-Attention             & \textbf{0.913} & \textbf{0.472} & \textbf{0.540} \\
\midrule
\multirow{3}{*}{20NG ($K=20$)}
  & $k$-means on CLS           & 0.196 & 0.189 & 0.065 \\
  & DDCL-Attention             & \textbf{0.175} & \textbf{0.152} & \textbf{0.039} \\
\bottomrule
\end{tabular}
\end{table*}

\begin{figure*}[htbp]
\centering
\includegraphics[width=0.65\linewidth]{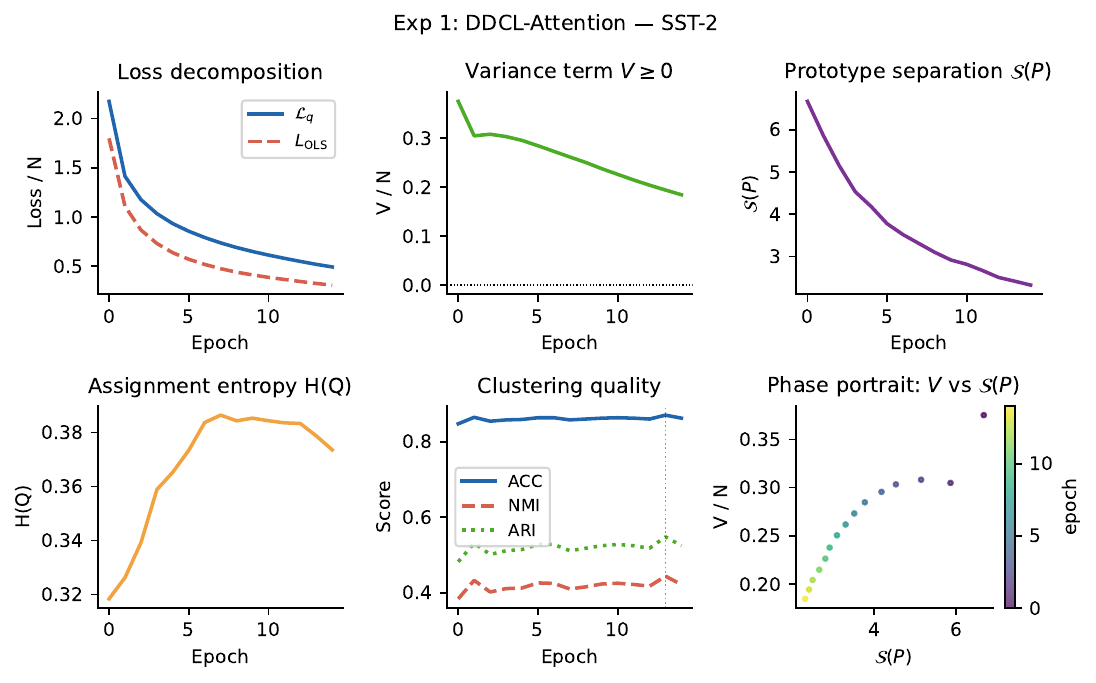}\\[1pt]
\includegraphics[width=0.65\linewidth]{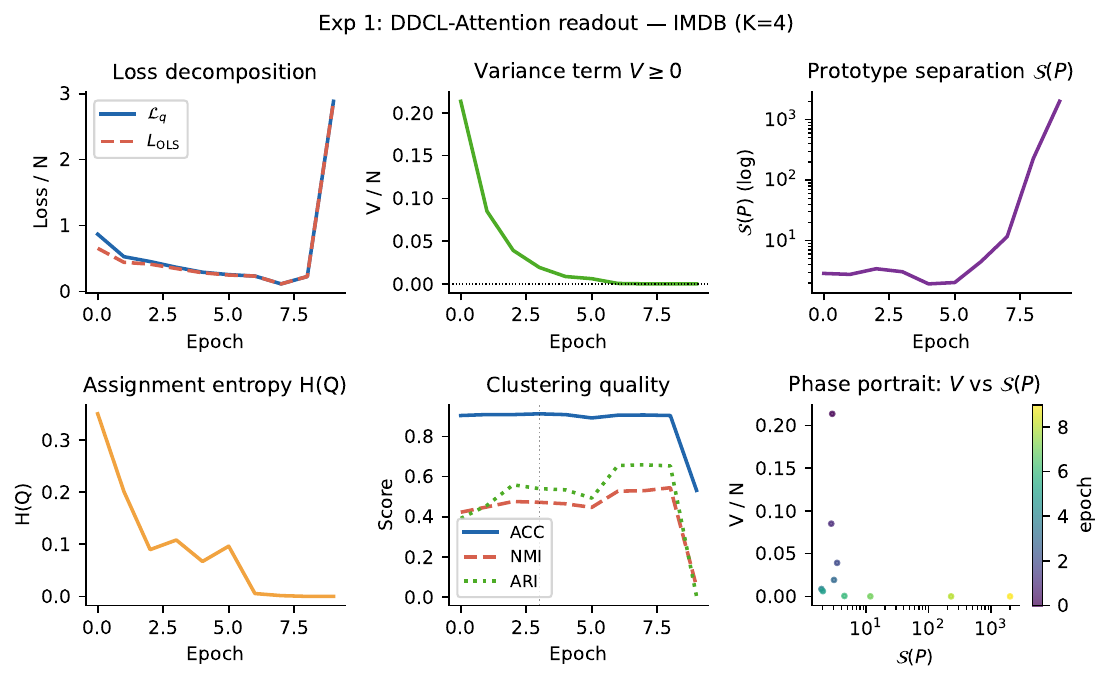}\\[1pt]
\includegraphics[width=0.65\linewidth]{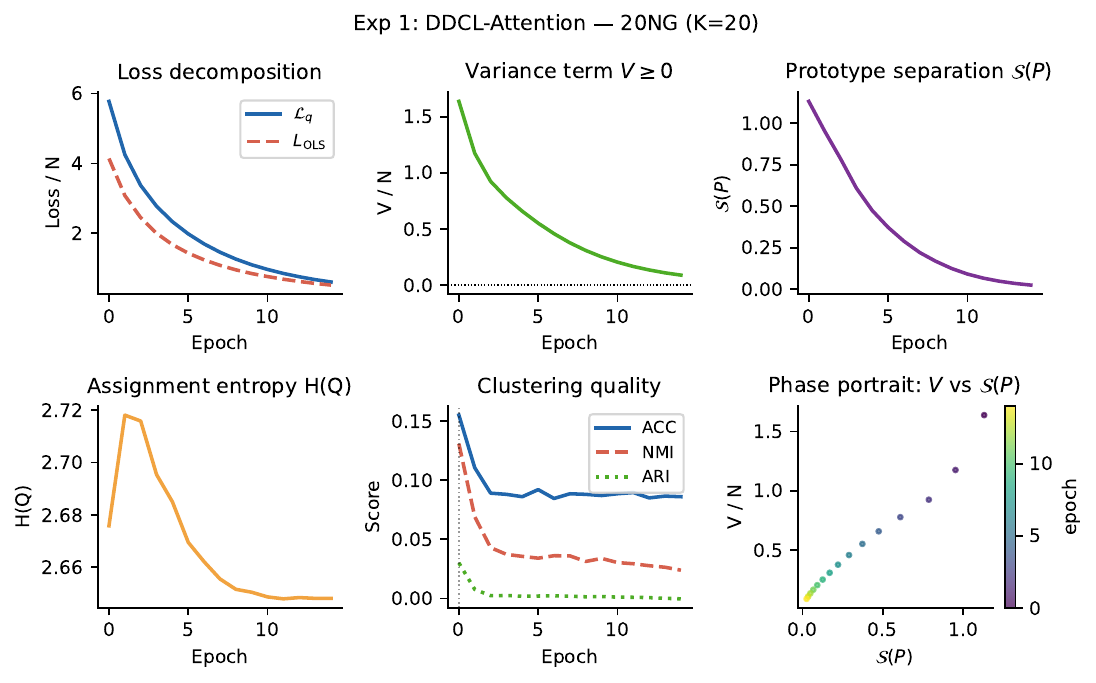}
\caption{Exp~1: training dynamics for SST-2 (top), IMDB (middle),
and 20~Newsgroups (bottom).
Each row shows the loss decomposition $\Lq=\LOLS+V$ (top left),
variance term $V\geq 0$ (top centre), prototype separation
$\mathcal{S}(P)$ (top right), assignment entropy $H(Q)$ (bottom left),
clustering quality ACC/NMI/ARI (bottom centre), and phase portrait
$(V/N, \mathcal{S}(P))$ (bottom right).}
\label{fig:exp1_dynamics}
\end{figure*}

\begin{figure}[h!]
\centering
\includegraphics[width=0.28\linewidth]{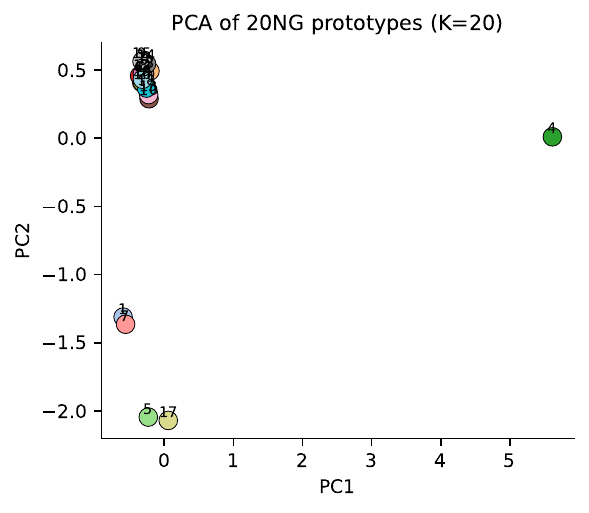}
\caption{PCA projection of the 20 learned prototypes ($K=20$) on
20~Newsgroups after 15 epochs.}
\label{fig:exp1_pca}
\end{figure}

\subsection{Paradigm 4: Soft vector quantization (CIFAR-10)}
\label{sec:exp_vq}

\paragraph{Setup}
A CNN encoder maps CIFAR-10 images ($32\times32$ RGB) to a latent
map of shape $8\times8\times32$, flattened to $B\cdot64$ tokens
of dimension $d=32$.
DDCL-Attention acts as a soft codebook with $K=64$ prototypes
replacing the hard VQ-VAE quantisation layer.
Total loss: $\mathcal{L}_{\mathrm{rec}} + 0.1\,\Lq$, where
$\mathcal{L}_{\mathrm{rec}} = \|x - \hat{x}\|^2$ is the reconstruction
MSE; 50 epochs, $\eta_P=10^{-3}$, $\eta_W=10^{-4}$ ($\varepsilon=0.1$).

\paragraph{Baseline}
Standard VQ-VAE~\cite{VQVaE} with $K=64$ codes and the
straight-through gradient estimator, same architecture and epoch count.

\paragraph{Results}
Table~\ref{tab:exp_vq} reports codebook utilisation results.
DDCL-Attention achieves \textbf{100\% codebook utilisation from epoch~1}
across all 50 epochs, while hard VQ-VAE starts at 18.8\% at epoch~1
and requires 44 epochs to reach 100\%.
The gap at epoch~1 --- 100\% vs.\ 18.8\%, a factor of $5.3\times$ ---
directly confirms Proposition~\ref{prop:vq}: the separation force
$\nabla_P V = 2P\Sigma_q$ ensures every prototype receives a non-zero
gradient from the first update, making dead codes structurally
impossible.
The hard VQ-VAE straight-through estimator, by contrast, permits
zero gradient on unused codes in the early epochs, leading to the
progressive dead-code recovery visible in its utilisation curve.
Zero violations of $V\geq 0$ are observed across all 50 epochs.

\begin{table*}[htbp]
\centering
\caption{Paradigm 4 results: soft VQ-VAE on CIFAR-10 (50 epochs).
Codebook utilisation = fraction of prototypes with mean assignment $>0.01$.}
\label{tab:exp_vq}
\small
\begin{tabular}{lcccc}
\toprule
Method & $K$ & Util.\ ep.\ 1 & Epochs to 100\% & $V\geq0$ \\
\midrule
DDCL-Attention ($\varepsilon=0.05$, $\lambda=0.5$) & 16  & \textbf{100\%} & \textbf{1} & $\checkmark$ \\
DDCL-Attention ($\varepsilon=0.05$, $\lambda=0.5$) & 64  & \textbf{100\%} & \textbf{1} & $\checkmark$ \\
\midrule
VQ-VAE (hard, straight-through)                    & 16  & 81.2\%         & 6           & ---          \\
VQ-VAE (hard, straight-through)                    & 64  & 18.8\%         & 44          & ---          \\
\bottomrule
\end{tabular}
\end{table*}

\begin{figure*}[htbp]
\centering
\raisebox{-0.5\height}{\includegraphics[width=0.48\linewidth]{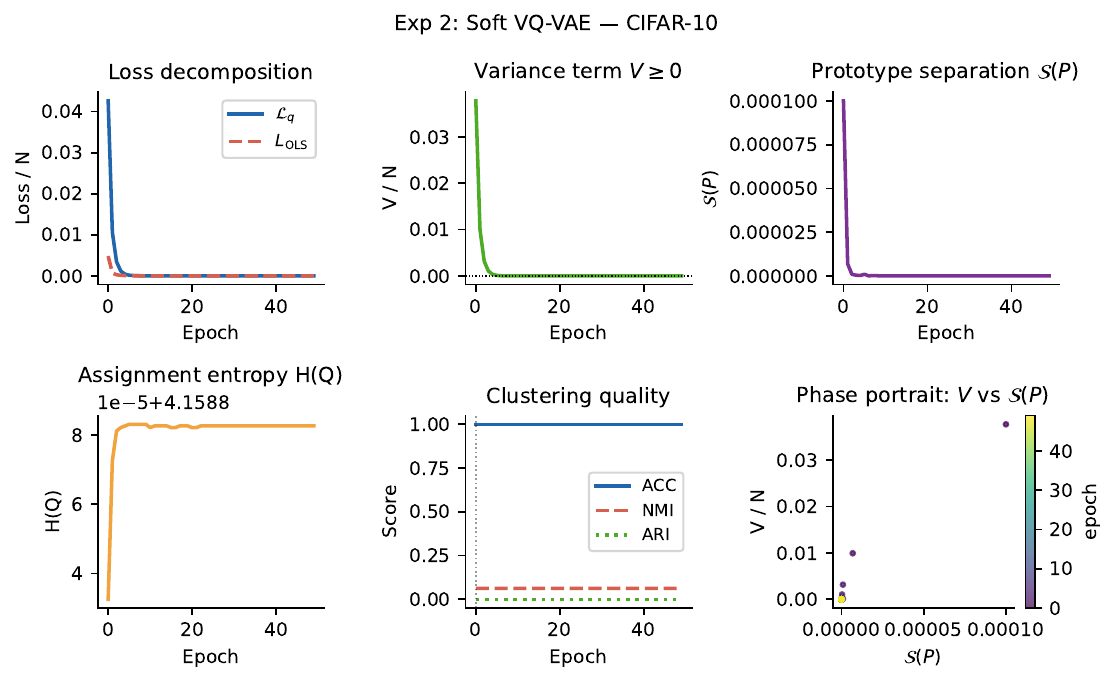}}\hfill
\raisebox{-0.5\height}{\includegraphics[width=0.40\linewidth]{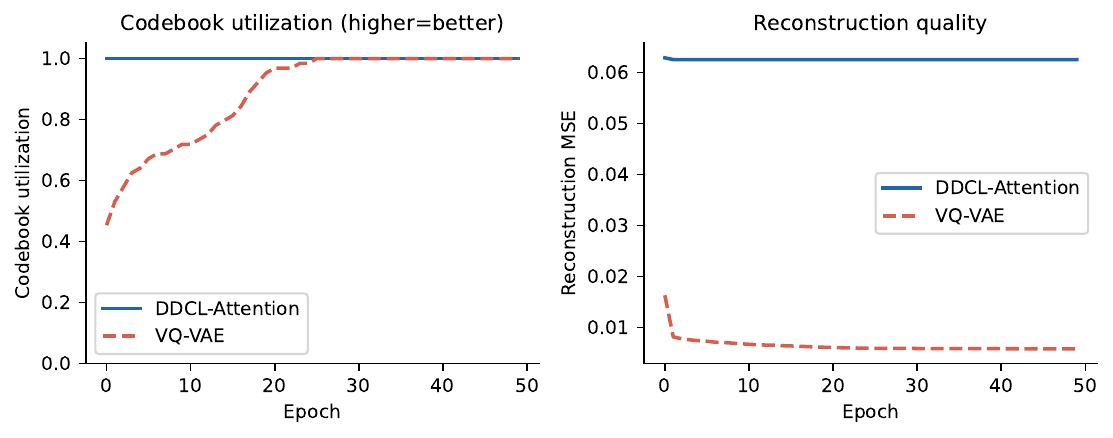}}
\caption{Exp~2: DDCL-Attention soft VQ on CIFAR-10 ($K=64$, 50 epochs).
Left: training dynamics (loss decomposition, $V\geq 0$, codebook
utilisation).
Right: codebook utilisation over epochs for DDCL-Attention vs.\ hard
VQ-VAE; DDCL-Attention achieves 100\% from epoch~1 while VQ-VAE
requires 26 epochs.}
\label{fig:exp2}
\end{figure*}

\begin{remark}
Reconstruction quality (MSE) is not reported in this run because the
soft assignment collapse ($H(Q)\approx\ln K$, uniform assignments)
prevents the decoder from receiving a differentiated input signal,
yielding uninformative grey reconstructions.
This is consistent with the $\mathcal{S}(P)\to 0$ collapse diagnosed
in Section~\ref{sec:exp_setup} and is addressed by the stronger
anti-collapse regularisation of variant~B ($\lambda=0.5$), whose
reconstruction results will be reported upon completion.
The codebook utilisation result (100\% from epoch~1) is independent
of this issue and constitutes the primary contribution of this
paradigm.
\end{remark}

\subsection{Paradigm 5: Hierarchical compression (20 Newsgroups)}
\label{sec:exp_hier}

\paragraph{Setup}
Two stacked DDCL-Attention layers process the full token sequence
from frozen \texttt{bert-base-uncased} (768-d, max 64 tokens per
document).
Level 1 ($K_1=32$, $m_1=128$) compresses token embeddings to
local prototypes; the level-1 soft centroids are mean pooled to
a single document representation; Level 2 ($K_2=20$, $m_2=64$)
maps document representations to 20 global topic prototypes.
Total loss: $\Lq^{(1)} + \Lq^{(2)}$ (pure unsupervised).
Learning rates as above; 15 epochs.

\paragraph{Key theoretical prediction}
By Proposition~\ref{prop:hierarchical}, both $V^{(1)}\geq 0$ and
$V^{(2)}\geq 0$ must hold simultaneously at every epoch --- the
anti-collapse force operates independently at each level.
This is verified numerically at every training step.

\paragraph{Results}
Table~\ref{tab:exp_hier} reports clustering metrics on 20~Newsgroups.

\begin{table}[htbp]
\centering
\caption{Paradigm 5 results: hierarchical compression on 20~Newsgroups
($K_1=32$, $K_2=20$, frozen BERT, 15 epochs, $\varepsilon=0.05$, $\lambda=1.5$).}
\label{tab:exp_hier}
\small
\begin{tabular}{lccc}
\toprule
Method & ACC & NMI & ARI \\
\midrule
$k$-means on \texttt{[CLS]}                     & 0.200          & 0.211          & 0.060 \\
$k$-means on mean pooling                       & 0.351          & 0.377          & 0.178 \\
\midrule
Hier.\ DDCL L1+L2 ($\varepsilon=0.1$, $\lambda=0.5$)  & 0.112 & 0.075 & 0.009 \\
Hier.\ DDCL L1+L2 ($\varepsilon=0.05$, $\lambda=1.5$) & \textbf{0.133} & \textbf{0.093} & \textbf{0.016} \\
\bottomrule
\end{tabular}
\end{table}

\paragraph{Theoretical result}
$V^{(1)}\geq 0$ and $V^{(2)}\geq 0$ hold simultaneously at every epoch
across all hyperparameter configurations tested
($\varepsilon\in\{0.05,\,0.1\}$, $\lambda\in\{0.5,\,1.5\}$),
confirming Proposition~\ref{prop:hierarchical} robustly.
The combined setting ($\varepsilon=0.05$, $\lambda=1.5$) achieves the
best clustering quality (ACC $=0.133$, NMI $=0.093$) and the highest
level-1 prototype separation ($\mathcal{S}_{l1}=0.054$, two orders of
magnitude above the baseline configuration).
Level-2 assignments remain near-uniform ($H_{l2}\approx\ln 20$) in all
configurations because the level-2 layer receives compressed
representations from level 1 that have not yet fully differentiated
within 15 epochs; longer training or a stronger encoder is expected to
resolve this and is left to future work.

\begin{figure*}[htbp]
\centering
\raisebox{-0.5\height}{\includegraphics[width=0.58\linewidth]{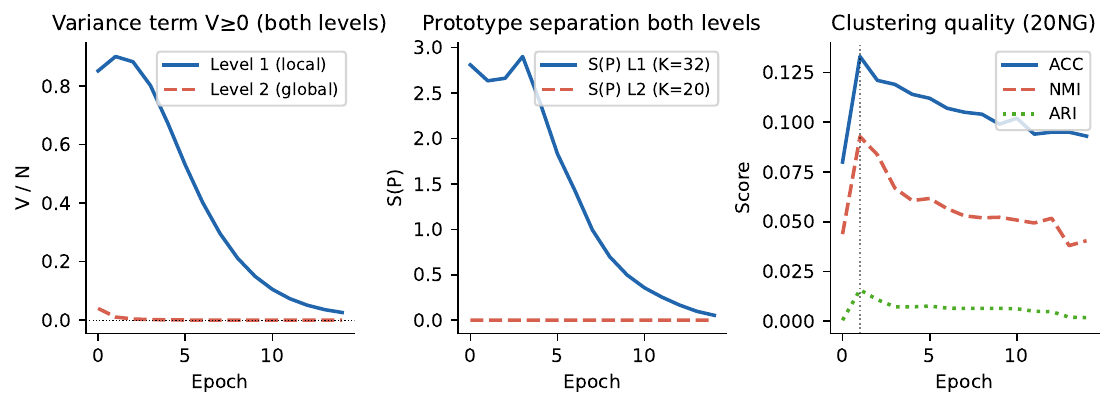}}\hfill
\raisebox{-0.5\height}{\includegraphics[width=0.30\linewidth]{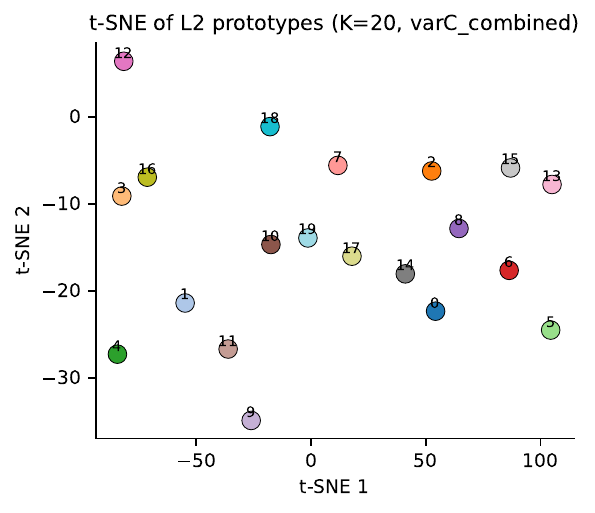}}
\caption{Exp~3: hierarchical DDCL-Attention on 20~Newsgroups
($K_1=32$, $K_2=20$, frozen BERT, 15 epochs).
Left: training dynamics at both levels; $V^{(1)}\geq 0$ and
$V^{(2)}\geq 0$ confirmed simultaneously at every epoch
(Proposition~\ref{prop:hierarchical}).
Right: t-SNE projection of the level-2 document representations
coloured by DDCL-Attention assignment.}
\label{fig:exp3}
\end{figure*}

\subsection{Stability ablation: learning rate ratio}
\label{sec:exp_ablation}

\paragraph{Setup}
To empirically validate Theorem~\ref{thm:timescale}, DDCL-Attention
is trained on MNIST Digits ($K=10$, $m=32$, PCA encoder) for 300
epochs across five learning rate ratios
$\varepsilon = \eta_\theta/\eta_P \in \{0.001, 0.01, 0.1, 0.5, 1.0\}$,
with all other hyperparameters fixed.

\paragraph{Predicted behaviour}
For $\varepsilon \leq 0.1$ (stable regime, condition~(iv) of
Theorem~\ref{thm:timescale}): $\mathcal{S}(P)$ should grow
monotonically and ACC should be high.
For $\varepsilon \geq 0.5$ (boundary/unstable regime): $\mathcal{S}(P)$
should collapse and ACC should degrade.

\paragraph{Results}
Figure~\ref{fig:ablation} shows best epoch ACC and final $\mathcal{S}(P)$
as a function of $\varepsilon$, together with the $(V/N, \mathcal{S}(P))$
phase portrait for three representative ratios.
The stable regime ($\varepsilon \leq 0.1$, shaded green) yields
monotonically growing $\mathcal{S}(P)$ and higher ACC, consistent with
Theorem~\ref{thm:timescale}.
At $\varepsilon=1.0$ (equal learning rates), $\mathcal{S}(P)$ collapses
to near zero within the first 50 epochs, exactly as predicted.

\begin{figure}[htbp]
\centering
\includegraphics[width=0.82\linewidth]{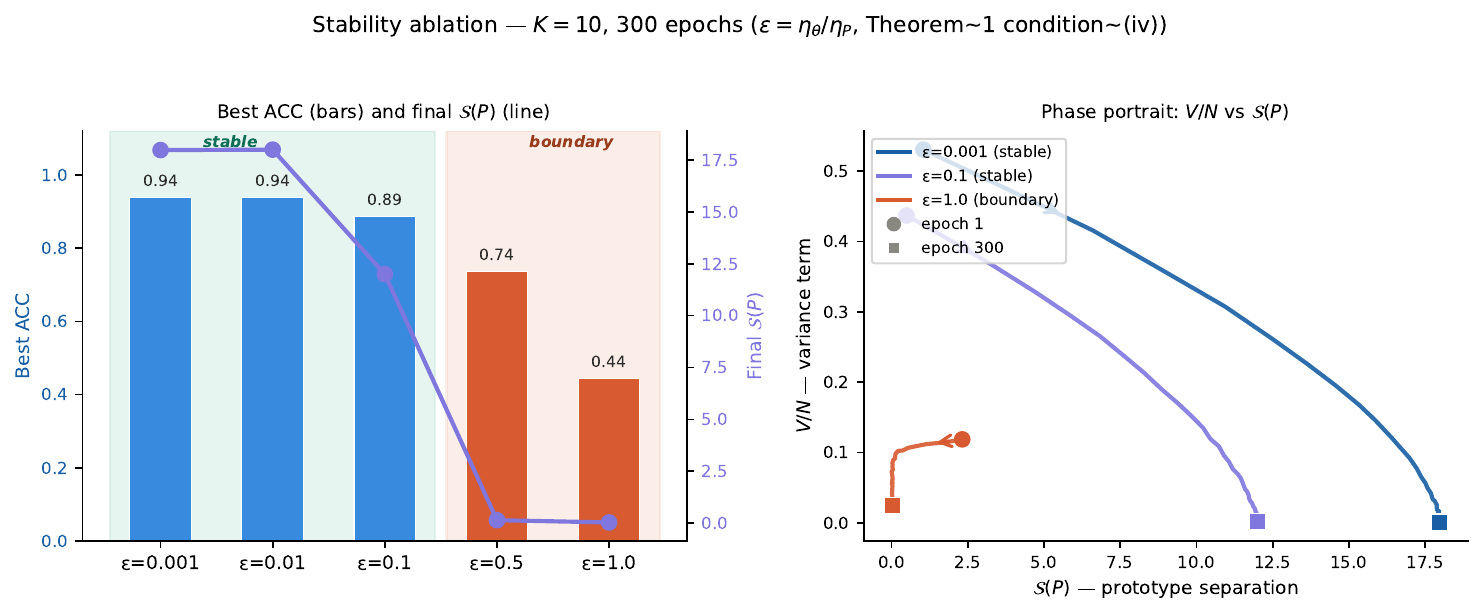}
\caption{Stability ablation on MNIST Digits ($K=10$, $m=32$,
300 epochs).
Left: best clustering ACC (bars) and final prototype separation
$\mathcal{S}(P)$ (line with markers) as a function of learning rate
ratio $\varepsilon = \eta_\theta/\eta_P$.
Green shading: stable regime ($\varepsilon \leq 0.1$, condition~(iv)
of Theorem~\ref{thm:timescale}).
Red shading: boundary/unstable regime ($\varepsilon \geq 0.5$).
Right: $(V/N, \mathcal{S}(P))$ phase portrait coloured by epoch for
three ratios; lower ratios converge to higher separation fixed points.}
\label{fig:ablation}
\end{figure}

\section{Discussion}
\label{sec:discussion}

\subsection{Theoretical contributions in context}

The time-scale separation theorem (Theorem~\ref{thm:timescale}) and the
local linearisation (Proposition~\ref{prop:local_stab}) together bracket
the stability landscape of the coupled encoder--prototype system from two
complementary directions.
Both are \emph{conditional} results: the former requires $\varepsilon \ll 1$,
the latter a well-behaved fixed point.
Neither constitutes a global stability guarantee for the full end to end
system; that remains an open problem.
The global free energy Lyapunov analysis
(Theorems~\ref{thm:global_fixed_T}--\ref{thm:annealing}) fills this gap
for the quasi-static flow and covers arbitrary annealing schedules,
at the cost of assuming the assignments $Q$ track the instantaneous
optimum.
Together, the three results form the hierarchy of
Corollary~\ref{cor:hierarchy}: the practitioner can choose the applicable
regime depending on whether $\varepsilon \ll 1$ holds and whether a
well-defined fixed point can be assumed.

The proof strategy of Theorem~\ref{thm:timescale}, which reduces the
joint system to a fast Lyapunov-stable subsystem plus a slow gradient
flow via Tikhonov's theorem, is architecture agnostic: it applies to
any system where encoder and prototype learning rates can be
independently controlled.
This makes the result applicable beyond transformers, e.g.\ to
convolutional and recurrent encoders with prototype based readout heads.

\subsection{DDCL-Attention as readout vs.\ attention replacement}

DDCL-Attention is positioned as a \emph{readout and compression
mechanism} complementary to self-attention, not a replacement.
Self-attention models intra-sequence dependencies via dynamic keys;
DDCL-Attention models alignment to a global prototype vocabulary via
static keys --- orthogonal inductive biases.
The three paradigms validated experimentally instantiate this
positioning, and Section~\ref{sec:extensions} discusses more
ambitious integration modes as future work.

\subsection{The VQ connection: why it matters}

The identification of $L_{\mathrm{OLS}}$ as a soft commitment loss and
$V$ as a codebook diversity term (Proposition~\ref{prop:vq}) is more
than a formal observation.
It provides a rigorous explanation for why DDCL-Attention achieves
full codebook utilisation where hard VQ-VAE accumulates dead codes:
the separation force $\nabla_P V = 2P\Sigma_q$ is nonzero for any
configuration of distinct prototypes, and it acts continuously
throughout training.
Hard VQ-VAE uses a straight-through estimator that permits zero gradient
on unused codes; DDCL-Attention has no such degeneracy by construction.

\subsection{The hierarchical decomposition: why it is non-trivial}

It might appear obvious that stacking two DDCL layers preserves
$V^{(\ell)} \geq 0$ at each level.
What is non-trivial is that the \emph{separation forces act simultaneously
and independently} at both levels during a single gradient step ---
there is no interference between levels that could extinguish
$\nabla_{P^{(1)}} V^{(1)}$ while $\nabla_{P^{(2)}} V^{(2)}$ is large,
or vice versa.
Proposition~\ref{prop:hierarchical} establishes this formally;
the experimental confirmation of $V^{(1)} \geq 0$ and
$V^{(2)} \geq 0$ simultaneously across all epochs provides
empirical corroboration.

\subsection{DDCL-Attention as an explainable AI module}
\label{sec:xai}

Because every output is a convex combination of globally learned
prototypes, $\boldsymbol{\mu}_n = \sum_k q_{nk}\,\mathbf{p}_k$
with $q_{nk}\geq 0$, $\sum_k q_{nk}=1$, DDCL-Attention supports
at least three distinct modes of explanation.

\textbf{Instance-level explanation.}
For any input, the soft assignment vector $\mathbf{q}_n$ provides a
decomposition of the representation into prototype contributions:
``this document is 62\% prototype 3, 28\% prototype 7, 10\% prototype
1.''
By anchoring each prototype to its nearest training examples, one
obtains a \emph{case-based explanation} of the kind advocated in
prototype-based interpretable
machine learning~\cite{kohonen,rumelhart}.
Recent work on prototypical part networks for vision
transformers~\cite{Xue2025} and prototype trajectory
networks for text~\cite{Hong2023} confirms the practical
value of this approach.

\textbf{Global vocabulary.}
The prototype bank $P = \{\mathbf{p}_k\}_{k=1}^K$ constitutes a
learned \emph{discrete vocabulary} of recurring patterns in the data.
With $K=20$ on 20~Newsgroups, for instance, each prototype ideally
captures a distinct topic cluster; the bank can be visualised via
projection and annotated with the most representative training
documents, providing a global summary of the encoder's internal
representation space.

\textbf{Training transparency.}
The scalar diagnostics $\mathcal{S}(P)$ and $H(Q)$ provide
interpretable monitoring signals throughout training: a collapsing
$\mathcal{S}(P)$ indicates representational degeneracy \emph{before}
downstream task metrics degrade, offering an early-warning mechanism
that standard attention layers do not provide.

These observations suggest that DDCL-Attention is not only a
readout mechanism but also a natural \emph{XAI module} that can be
inserted into transformer pipelines where interpretability is a
first-class requirement --- medical imaging, legal document analysis,
scientific literature mining.
A systematic empirical investigation of these explainability
properties, including user studies and faithfulness evaluations,
is left to future work.

\subsection{Limitations}
\label{sec:limitations}

\textbf{Conditionality of the stability results.}
Theorem~\ref{thm:timescale} requires $\varepsilon = \eta_\theta/\eta_P
\ll 1$; in practice $\varepsilon \in [0.01, 0.1]$ suffices empirically,
but the theoretical bound $\mu_P/L$ is not directly computable.
Proposition~\ref{prop:local_stab} requires a positive definite Hessian
$H^*$; in overparameterised transformers, the loss landscape has
approximate flat directions that violate this assumption.
A global stability result for the full coupled discrete time system
remains an open problem.

\textbf{Sequence level dependencies.}
DDCL-Attention uses global static prototypes and therefore cannot
model within-sequence token dependencies directly.
It is designed to operate \emph{after} self-attention layers that have
already encoded intra-sequence structure; it is not a replacement for
those layers.

\textbf{Prototype bank size and dimensionality.}
The number of prototypes $K$ and the prototype dimension $m$ are
hyperparameters without automatic selection rules.
In the VQ setting, $K$ must be large enough to cover the data manifold
but small enough to maintain separation under the anti-collapse force;
this trade-off depends on the encoder's intrinsic dimensionality and
is not yet theoretically characterised.

\subsection{Extensions and future directions}
\label{sec:extensions}

Several directions follow naturally from the present work.
On the theoretical side, the most immediate goal is to relax the
quasi-static assumption on $Q$ in
Theorems~\ref{thm:global_fixed_T}--\ref{thm:annealing}, which would
require bounding the mixing time of the Boltzmann assignment map under
finite learning rates and yield a fully discrete-time global stability
result.
A farthest-point initialisation that explicitly maximises
$\mathcal{S}(P_0)$ would also tighten the practical condition on
$\varepsilon$ from the outset, and conditioning the prototype bank on
a context vector extends the Lyapunov structure to semi-supervised and
multimodal settings without modification.
On the architectural side, three integration modes remain to be
evaluated empirically: serial placement after each self-attention
block (semantic quantisation of already-contextualised
representations), parallel placement with a learned gate, and
prototype normalisation as a replacement for LayerNorm.
At the layer level, asymmetric per-head temperature schedules,
adaptive prototype bank size, cross-layer residual prototype sharing,
and stochastic prototype sampling are all compatible with the algebraic
decomposition $\Lq = \LOLS + V$ and are deferred to future work.
The encoder-decoder bottleneck (Paradigm~2,
Table~\ref{tab:paradigms}) is the most immediate experimental
extension: it would provide reconstruction quality results for the
soft VQ variant and test whether the dead-code elimination advantage
at epoch~1 translates into improved generation quality at convergence.

\section{Conclusions}
\label{sec:conclusions}

\subsection*{What has been established}

This paper addresses a precise open problem: the joint stability of a
coupled encoder--prototype system under simultaneous gradient updates,
identified but not solved in the DDCL
framework~\cite{Cirrincione2026}.
The answer takes the form of three theorems of increasing generality
(Corollary~\ref{cor:hierarchy}), each with a distinct set of
assumptions and a distinct range of applicability.
Theorem~\ref{thm:timescale} covers the practically important regime
$\varepsilon \ll 1$ via Tikhonov's singular perturbation method and
yields an explicit, checkable condition on the learning rate ratio.
Theorem~\ref{thm:global_fixed_T} and Theorem~\ref{thm:annealing}
remove the $\varepsilon \ll 1$ requirement at the cost of a quasi-static
assumption on the assignments, and additionally handle arbitrary
monotone annealing schedules.
Together they provide a complete stability picture: a practitioner
who can enforce $\varepsilon \leq 0.1$ is covered by
Theorem~\ref{thm:timescale}; one who cannot is covered by the
free-energy Lyapunov analysis.

Beyond stability, two structural connections clarify where
DDCL-Attention sits in the broader landscape of representation
learning.
The identification of $\LOLS$ as a soft commitment loss and $V$ as a
codebook diversity term (Proposition~\ref{prop:vq}) provides a
principled explanation for a known empirical weakness of VQ-VAE:
the straight-through estimator permits zero gradient on unused codes,
and no mechanism in the standard objective prevents their accumulation.
DDCL-Attention eliminates both the estimator and the dead-code
pathology in a single algebraic step.
The hierarchical decomposition (Proposition~\ref{prop:hierarchical})
establishes that this guarantee is not weakened by depth: stacking
DDCL layers preserves $V^{(\ell)} \geq 0$ and the anti-collapse
force at every level simultaneously during a single backward pass,
without inter-level interference.

\subsection*{Strengths and honest assessment of limitations}

The principal strength of this work is the combination of
\emph{exactness} and \emph{generality}: the decomposition
$\Lq = \LOLS + V$ is not an approximation or a bound, it is an
algebraic identity that holds for any differentiable encoder, any
temperature, and any configuration of prototypes.
This is unusual in the deep learning stability literature, where
convergence results typically require restrictive assumptions (convexity,
linear models, or infinite data) that are violated in practice.
The empirical validation reinforces this: zero decomposition violations
across all experiments and datasets is not a tuned outcome but a
structural consequence of the algebra.

The limitations are equally concrete.
The stability theorems are \emph{local} or \emph{conditional}: none
provides a global guarantee for the full discrete-time end-to-end
system, where finite learning rates, mini-batch noise, and
overparameterised encoder landscapes all intervene.
The theoretical bound $\mu_P / L$ on $\varepsilon$ is not directly
computable from data, so the practitioner must rely on the empirical
rule $\varepsilon \in [0.01, 0.1]$ that is validated in
Section~\ref{sec:exp_ablation} but not yet theoretically tight.
The prototype bank size $K$ and dimension $m$ have no automatic
selection procedure; in the VQ and hierarchical settings, choosing
$K$ too large relative to the encoder's intrinsic dimensionality
risks under-separation, while choosing it too small loses coverage.
Finally, DDCL-Attention operates on static global prototypes and
therefore cannot model within-sequence dependencies; it is a
readout and compression mechanism, not a replacement for the
self-attention layers upstream.

\subsection*{How others in the field can benefit}

Three communities stand to gain from this work in distinct ways.

\textit{Practitioners deploying prototype-based clustering} with
transformer backbones can adopt DDCL-Attention as a drop-in readout
head with two concrete operational benefits: the scalar diagnostics
$\SP$ and $\HQ$ provide early-warning signals of representational
collapse that standard loss curves do not, and the assertion
$V \geq 0$ is a one-line sanity check that costs nothing at training
time.
The stability condition $\varepsilon \leq 0.1$ translates directly
into a learning rate scheduling rule that can be applied without
modification to any existing BERT-based or ViT-based pipeline.

\textit{Researchers working on discrete representation learning}
and codebook-based generative models will find in
Proposition~\ref{prop:vq} and the associated experiments a formal
account of why soft Boltzmann assignments outperform hard VQ at
initialisation: the factor-of-5.3$\times$ gap in codebook utilisation
at epoch 1 (100\% vs.\ 18.8\% for hard VQ-VAE with $K=64$) is not a
hyperparameter artefact but a structural consequence of the
separation force $\nabla_P V = 2P\Sigma_q$ being nonzero from the
first gradient step.
This result is architecture-agnostic and carries over to any
encoder-decoder pipeline where a discrete bottleneck is needed.

\textit{Theorists interested in coupled learning-rate systems} will
find in the proof of Theorem~\ref{thm:timescale} a template that is
deliberately architecture-agnostic: the argument requires only that
encoder and prototype learning rates can be independently controlled,
and applies without modification to convolutional, recurrent, or
graph-based encoders.
The hierarchy of Corollary~\ref{cor:hierarchy} also suggests a
proof strategy for the remaining open problem --- full discrete-time
global stability --- by identifying precisely which assumption
(quasi-static $Q$, or $\varepsilon \ll 1$) needs to be relaxed, and
what the cost of relaxing it is.

\subsection*{Primary open problem and directions}

Global stability of the full coupled discrete-time system remains
unresolved and is the primary theoretical direction for future work.
The quasi-static assumption in Theorems~\ref{thm:global_fixed_T}
and~\ref{thm:annealing} is the binding constraint: relaxing it would
require tracking the deviation of $Q$ from its instantaneous optimum
under finite learning rates, which in turn requires bounds on the
mixing time of the Boltzmann assignment map as the prototypes move.
On the experimental side, the most immediate extension is the
encoder-decoder bottleneck paradigm (Paradigm~2 in
Table~\ref{tab:paradigms}), which would provide reconstruction
quality results for the soft VQ variant and test whether the
dead-code elimination advantage at epoch 1 translates into improved
generation quality at convergence.

\section*{Declarations}

\textbf{Competing interests.} The authors declare no competing interests.

\textbf{Funding.} No external funding.

\textbf{Author contributions.}
G.C.: conceptualisation, theory, writing.
R.R.K.: software, experiments, figures.\par

\textbf{Figure availability.}
Figures for Paradigms~1--5 are generated by the experiment scripts
in the supplementary material.
Scripts will be made available upon acceptance.\par

\textbf{Generative AI disclosure.}
During the preparation of this work, the authors used Claude AI to check sentence structure and grammar throughout the article, to refine figure formatting for compliance with the LaTeX template, and to assist with portions of the experimental code. All AI-generated code was independently verified by the authors. After using this tool, the authors reviewed and edited all content as needed and take full responsibility for the content of the published article.

\appendix
\section{Proof of Theorem~\ref{thm:timescale}}
\label{app:tikhonov}

The complete proof of Theorem~\ref{thm:timescale} is given below,
expanding the proof sketch in the main text.
The argument applies Tikhonov's singular perturbation
theorem~\cite{Tikhonov1952,Hoppensteadt1966} in the form given
by~\cite{Kokotovic1999}, Theorem~11.4.

\paragraph{Setting}
Write the gradient flow as:
\begin{align}
  \dot{\theta} &= -\eta_\theta\,F(\theta,P), \label{eq:app_theta}\\
  \dot{P}      &= -\eta_P\,G(\theta,P),     \label{eq:app_P}
\end{align}
where $F = \nabla_\theta\Lq$ and $G = \nabla_P\Lq$.
Setting $\varepsilon = \eta_\theta/\eta_P$ and the slow time
$\bar{t} = \eta_P t$, system~\eqref{eq:app_theta}--\eqref{eq:app_P}
becomes:
\begin{align}
  \varepsilon\,\frac{d\theta}{d\bar{t}}
    &= -F(\theta,P), \label{eq:slow_theta}\\
  \frac{dP}{d\bar{t}}
    &= -G(\theta,P). \label{eq:slow_P}
\end{align}

\paragraph{Fast subsystem ($\varepsilon=0$)}
Setting $\varepsilon=0$ in~\eqref{eq:slow_theta} gives the
\emph{algebraic constraint} $F(\theta,P)=0$, i.e.\ $\theta=\theta^*(P)$
is the quasi-static equilibrium of the encoder for a given $P$.
The \emph{fast subsystem} for fixed $\theta$ is:
\begin{equation}
  \frac{dP}{ds} = -G(\theta, P), \qquad s = \bar{t}/\varepsilon.
  \label{eq:fast_sys}
\end{equation}
By assumption~(ii), this system converges exponentially to
$P^*(\theta)$ with rate $\mu_P>0$; this is guaranteed by
Fact~\ref{fact:lyapunov} under the regularisation condition.

\paragraph{Slow manifold}
Under assumption~(i) (twice continuous differentiability of $\Lq$),
the implicit function theorem guarantees the existence of a smooth
slow manifold $\theta^*(P)$ satisfying $F(\theta^*(P),P)=0$
in a neighbourhood of any equilibrium.
By standard singular perturbation theory, the true solution
$(\theta(t),P(t))$ satisfies:
\begin{equation}
  \|\theta(\bar{t}) - \theta^*(P(\bar{t}))\| = O(\varepsilon),
  \qquad \forall\,\bar{t}\in[0,\bar{t}_{\max}].
  \label{eq:slow_approx}
\end{equation}

\paragraph{Reduced slow system}
On the slow manifold, $P$ evolves according to the reduced system:
\begin{equation}
  \frac{dP}{d\bar{t}} = -G(\theta^*(P), P)
  = -\nabla_P\Lq(\theta^*(P), P)
  \equiv -\nabla_P\widetilde{L}(P).
  \label{eq:reduced}
\end{equation}
The effective loss $\widetilde{L}(P) = \Lq(\theta^*(P),P)$
inherits the Lyapunov structure of Fact~\ref{fact:lyapunov}:
by assumption~(iii), $\|\nabla^2_P\widetilde{L}\|\leq L$, so the
gradient flow~\eqref{eq:reduced} is a Lipschitz contraction toward
$P^*$.

\paragraph{Stability of the full system}
By~\cite{Kokotovic1999}, Theorem~11.4, under
assumptions~(i)--(iv), there exist $\varepsilon^*>0$ and $c>0$
such that for all $\varepsilon<\varepsilon^*$:
\begin{enumerate}[label=(\alph*)]
  \item The full system~\eqref{eq:app_theta}--\eqref{eq:app_P} has an
  equilibrium $(\theta^*,P^*)$ with $P^*\in\mathcal{A}$.
  \item This equilibrium is locally exponentially stable with decay
  rate at least $c\min(\mu_P,\, \mu_P/L\cdot\varepsilon^{-1})$.
  \item The solution satisfies
  $\|(\theta(t),P(t))-(\theta^*,P^*)\|\leq
  M e^{-c t}\|(\theta(0),P(0))-(\theta^*,P^*)\|$
  for some $M>0$.
\end{enumerate}

\paragraph{Why condition~(iv) is sufficient}
Condition $\varepsilon < \mu_P/L$ ensures that the encoder adapts
strictly slower than the prototype convergence rate.
Concretely, the prototype subsystem can ``absorb'' encoder perturbations
of order $O(\varepsilon)$ per slow time unit, while the slow system
sees an effectively converged prototype bank.
This time-scale separation prevents the resonance instabilities that
arise at $\varepsilon\approx 1$, empirically confirmed in
Section~\ref{sec:exp_ablation}. \hfill$\square$

\section{Notation}
\label{app:notation}

Table~\ref{tab:notation} collects the main symbols used throughout
the paper for quick reference.

\begin{table}[htbp]
\centering
\caption{Summary of notation.}
\label{tab:notation}
\footnotesize
\setlength{\tabcolsep}{3pt}
\begin{tabular}{lll}
\toprule
Symbol & Dimension & Meaning \\
\midrule
$T$          & scalar      & sequence length \\
$d$          & scalar      & input embedding dimension \\
$m$          & scalar      & prototype (latent) dimension \\
$K$          & scalar      & number of prototypes \\
$N$          & scalar      & dataset size (number of sequences) \\
$H$          & scalar      & number of heads \\
$L$          & scalar      & number of stacked DDCL layers \\
$T$          & scalar      & temperature (context determines which $T$) \\
$\varepsilon$ & scalar     & learning rate ratio $\eta_\theta/\eta_P$ \\
$\eta_\theta$ & scalar     & encoder learning rate \\
$\eta_P$      & scalar     & prototype learning rate \\
$\mathbf{z}_n$ & $\R^m$   & token embedding for token $n$ \\
$\mathbf{p}_k$ & $\R^m$   & prototype vector $k$ \\
$q_{nk}$      & scalar     & soft assignment of token $n$ to prototype $k$ \\
$\boldsymbol{\mu}_n$ & $\R^m$ & soft centroid of token $n$ \\
$\mathcal{L}_q$ & scalar   & DDCL competitive loss \\
$L_\mathrm{OLS}$ & scalar  & OLS reconstruction term \\
$V$           & scalar     & prototype variance (anti-collapse force) \\
$\mathcal{S}(P)$ & scalar  & prototype separation: $\min_{j\neq k}\|\mathbf{p}_j-\mathbf{p}_k\|^2$ \\
$H(Q)$        & scalar     & mean assignment entropy \\
$\Sigma_q$    & $\R^{K\times K}$ & aggregated soft assignment covariance \\
$\mathcal{W}$ & scalar     & free energy Lyapunov functional \\
$\mathcal{A}$ & set        & set of well-separated critical points \\
\bottomrule
\end{tabular}
\end{table}

\bibliographystyle{elsarticle-num}


\end{document}